\documentclass[onecolumn]{IEEEtran}
\IEEEoverridecommandlockouts                              

\usepackage{amsmath,amssymb}
\usepackage{graphicx}
\usepackage{amsthm}

\newtheorem{lemma}{Lemma}
\newtheorem{theorem}{Theorem}

\newtheorem{assumption}{Assumption}
\newtheorem{corollary}{Corollary}

\usepackage{graphicx}
\usepackage{booktabs} 

\usepackage{hyperref}
\usepackage[noadjust]{cite}

\usepackage[noend]{algorithmic}
\usepackage[linesnumbered,ruled,vlined,algo2e]{algorithm2e} %
\usepackage{algorithm}
\let\oldnl\nl
\newcommand{\nonl}{\renewcommand{\nl}{\let\nl\oldnl}}

\usepackage{color}
\definecolor{red}{rgb}{1,0.2,0.2}
\definecolor{green}{rgb}{0.2,1,0.5}
\definecolor{blue}{rgb}{0,0,0}
\definecolor{lightblue}{rgb}{0.3,0.5,1}

\newcommand{\blue}[1]{\textcolor{blue}{#1}}
\usepackage[linesnumbered,ruled,vlined,algo2e]{algorithm2e} %
\newcommand{\inprod}[2]{\left< #1, #2 \right>}
\newcommand{\Ept}{{\mathbb E}}

\title{\LARGE \bf Federated Learning with Server Learning: Enhancing Performance for Non-IID Data}
\author{Van Sy Mai, Richard J. La, Tao Zhang\thanks{V. S. Mai and T. Zhang are with the National Institute of Standards and Technology (NIST), Gaithersburg, MD 20899, USA. Email: \{vansy.mai, tao.zhang\}@nist.gov. R.J. La is with NIST and the University of Maryland, College Park, MD 20742, USA. Email: hyongla@umd.edu. 
Any mention of commercial products in this paper is for information only; it does not imply any recommendation or endorsement by NIST.  U.S. Government work not protected by U.S. copyright}}

\begin{document}

\maketitle
\begin{abstract}
Federated Learning (FL) has emerged as a means of distributed learning using local data stored at clients with a coordinating server. Recent studies showed that FL can suffer from poor performance and slower convergence when training data at clients are not independent and identically distributed. Here we consider a new complementary approach to mitigating this performance degradation by allowing the server to perform auxiliary learning from a small dataset. Our analysis and experiments show that this new approach can achieve significant improvements in both model accuracy and  convergence time even when the server dataset is small and its distribution differs from that of the aggregated data from all clients. 
\end{abstract}

\section{Introduction}

Federated Learning (FL) is a recent paradigm in which multiple clients collaborate under the coordination of a central server to train machine learning (ML) models \cite{Kairouz2019}. A key advantage is that clients need not send their local data to any central sever or share their data with each other. Performing learning where the data is generated (or collected) is becoming necessary as a large and growing amount of data is created at the network edge and cannot all be forwarded to any central location due to many factors such as network capacity constraints, latency requirements, and data privacy concerns \cite{MungZhang2016}.

In its basic form, FL trains a global model for all clients based on the following high-level iterative procedure. At each global round: 1) the central server selects a subset of clients and shares the current global model with them, 2) each selected client updates the model using only its local data and forwards the updated model to the central server, and 3) the central server aggregates the updated local models from the clients to update the global model. This process is repeated until certain convergence criteria are satisfied. 

\underline{\bf Background:} Conventional FL techniques, such as the well-known Federated Averaging (FedAvg) algorithm \cite{McMahan2017}, carry out model aggregation by averaging the model parameters received from the clients. This performs well when clients have access to independent and identically distributed (IID) training samples. In practice, however, the local data available to the clients often do not satisfy this IID assumption for different reasons. For instance, clients may collect data from different sources, using different tools, 
under different conditions, or only have access to partial or biased data, which can cause the distributions of the samples or features at different clients to differ considerably. Such divergences are also referred to as {\em drifts} or {\em shifts}, and can take different forms \cite{Kairouz2019}. 

Large divergences can cause conventional FL techniques to suffer from poor model performance and slow training convergence \cite{Deng2020, Haddadpour2019, Jiang2019, Karimireddy2020, Li2020, Zhao2018}. For example, feature divergence, where the distributions of features differ at different clients, may cause local models to focus on different features or even use different feature representations. 
Non-IID training data can also cause clients to optimize their local models toward local optima that can differ significantly from global optima. This can further cause the weights of clients' local models to diverge \cite{Malinovsky2020, Zhao2018}. As a result, simply averaging local models may not move the aggregated model toward a global optimum. 

Recently, growing efforts have been devoted to improving FL performance for non-IID data. The following are several representative categories of approaches.

$\bullet$ Personalization: Clients personalize their local models to perform well on their local data \cite{Deng2020, Rad2021, Kulkarni2020, DLi2019, Xie2021}. Personalization can be for individual clients or groups of clients (e.g., clients that have similar training data or contribute similar model updates to the server) \cite{Briggs2020} \cite{Huang2021}. Many real-world applications, however, desire a common model for all clients. For example, consider autonomous vehicles (AVs) in different geographical regions learning to recognize stop signs. The snow-covered stop signs in northeast United States can look very different from those along the sunny southern country roads. Since cars can travel anywhere, they will benefit from a model that can work well everywhere. 

$\bullet$ Changing how clients learn or contribute: Several approaches aim to better align the objectives of clients that can diverge due to non-IID training data, e.g.,~\cite{Shoham2019, Zhang2021}. Clients may use Batch Normalization to alleviate local model divergence caused by non-IID data \cite{Li2021}. Batch Normalization \cite{Ioffe2015} has been used in deep learning to mitigate the impact of domain shifts (i.e., differences between training data distribution and test data distribution). 
Various methods have also been proposed to choose a subset of the clients to participate in each round of FL to counterbalance distribution shifts  \cite{Wang2020, WZhang2021}. 

$\bullet$ Changing how the server aggregates local models: This approach alters the aggregation method of local models based on, e.g., their distances to an estimated global model baseline \cite{Yeganeh2020}, or additional client states or control variates \cite{Karimireddy2020}.


$\bullet$ Lifelong learning techniques: These techniques treat the learning at each client as a separate task and learn these tasks sequentially using a single model without forgetting the previously learned tasks \cite{Kairouz2019}. 

\underline{\bf Motivation:} Our main observation is that these existing FL algorithms do not consider the central server as a learner or assume that the server has no training data. In practice, however, the server can and often have access to some training data. For example, the server may receive data from sensors and testing devices that do not participate in the learning process. It may have synthetic data obtained from simulation (or emulation) and digital twins. The server may also receive some raw data directly from the clients; this is often required to, for example, support system monitoring and diagnosis. 

Consider again AVs, as an example, which need ML models to recognize objects. Today, two main sources of data are used to train and test such models. First, test vehicles are used to scout selected areas to collect real-world data. Note that this typically imposes no privacy concerns. It, however, may require large fleets of test vehicles, take years to accomplish, incur heavy costs, and yet still fail to collect enough data to cover the vast range of possible learning needs \cite{Zhang2020}. Therefore, the AV industry is increasingly relying on a second source of data -- synthetic data, typically generated in the cloud -- to extend model training and testing scopes. 
Going forward, when some AVs participate in FL, a small fleet of test vehicles, which may not all participate in FL, can still be used to collect and send data to the server to compensate the data that the FL clients can collect.



Sharing a common IID training dataset with all clients (so that each client will train its local model on its local data plus this common dataset) has been shown to improve FL performance with non-IID data \cite{Kairouz2019, DLi2019, Zhao2018}. But, this method, which we refer to as {\it FL with data sharing} or simply {\it data sharing}, also increases clients' workload, making them less suitable for resource-constrained devices. More importantly, it is often impractical for clients to share data with each other due to privacy concerns, network bandwidth constraints, and latency requirements. We will show that it is unnecessary to share such common datasets among clients, as comparable or better performance can be achieved by having the server learn from the same dataset.

\blue{Several recent works have considered server learning with some centralized data, e.g., hybrid training \cite{yang2022convergence}, mixed FL \cite{augenstein2022mixed}, and FL with server learning~\cite{mai2023ciss}. However, \cite{yang2022convergence} analyzes only the case where both clients' data and server data are IID and their algorithm requires all clients to participate in every round. Similarly, \cite{augenstein2022mixed} assumes IID client data and considers server's role as a regularizer. In contrast, \cite{mai2023ciss} focuses on FL with non-IID client data. In this paper, we build upon our work in \cite{mai2023ciss} to study the idea of using server learning to enhance FL  on non-IID data and provide both analytical and experimental results showing that this approach can be effective under certain conditions. Therefore, the primary focus of our study and reported analysis are fundamentally different from those in \cite{augenstein2022mixed} and \cite{yang2022convergence}.}

\underline{\bf Contributions:} 
We consider a new FL algorithm that incorporates server learning to improve performance on non-IID data. Specifically, the server collects a small amount of data, learns from it, and distills the knowledge into the global model incrementally during the FL process. We refer to this method as {\em Federated Learning with Server Learning} (FSL). 
Our main contributions can be summarized as follows:
 

$\bullet$ Through our analysis and experimental studies, we show that FSL can significantly improve the performance in both final accuracy and  convergence time when clients have non-IID data. Also, only a {\it small amount of data} is needed at the server for FSL to improve performance, even when the server data distribution deviates from that of the aggregated data stored at the clients. As expected, the training performance improves as such distribution divergence diminishes. 

$\bullet$ By incorporating server learning with FL in an \textit{incremental fashion}, we will demonstrate
that FSL significantly accelerates the learning process when the current model is far from any (locally) optimal model.

$\bullet$  FSL is simple and can be tuned relatively easily, even when the server dataset is relatively small. Compared to FL, FSL adds only a local learning component to the server and does not affect the clients. Thus, FSL has the same per-round communication overhead as FL while practically requiring to tune only one additional parameter, which is the weight given to server's loss function. Our experimental studies show that the performance improvement of FSL remains significant for a relatively large range of this weight. 

In our experiments, FSL consistently outperforms the data sharing method in \cite{Zhao2018}, suggesting that sharing common datasets with clients might be unnecessary. We also demonstrate that by employing a small amount of data from either a few clients or other data sources (including synthetic data) for server learning, FSL can achieve similar (and often better) performance compared to FedDyn \cite{durmus2021federated} and SCAFFOLD \cite{Karimireddy2020}, while enjoying a significant boost in learning rate at the beginning.

\blue{Preliminary results of this paper appeared in \cite{mai2023ciss}, where only the main algorithm and limited experimental results using \underline{IID server data} were reported. In this paper, we provide a theoretical analysis of FSL and more extensive experimental evaluations, including a comparison with SCAFFOLD algorithm using \underline{non-IID server data}.}


The rest of the paper is organized as follows. The problem formulation and our algorithm are given in \S~\ref{sec_prob}. Main convergence results are presented in \S~\ref{sec:Convergence}, followed by experimental evaluations in \S~\ref{sec_experiment}. Conclusions are given in \S~\ref{sec_conclusion}. All the proofs and additional numerical results can be found in our technical report in 
Appendices~\ref{sec_Proofs} and~\ref{sec_further_result}, respectively. 

\textbf{Notation}: For each integer $n>0$, we use $[n]$ to denote the set $\{1, \ldots, n\}$. For a finite set $\mathcal{D}$, $|\mathcal{D}|$ denotes its cardinality. For any vector $x$, $\|x\|$ denotes its 2-norm. We denote by $\langle x, y \rangle$ the inner product of two vectors $x$ and $y$. A function $f: D \to \mathbb{R}$ is said to be smooth with parameter $L$, or simply $L$-smooth, if $f(x) - f(y) - \langle \nabla f(y), x-y \rangle  \le \frac{L}{2}\| x-y \|^2$ for all $x, y \in D$. For a random variable $X$, we use both $\Ept[X]$ and $\Ept X$ to denote its expected value.


\section{Problem Formulation and Our Approach}\label{sec_prob}
In this section, we first present our problem formulation in connection with the data sharing approach, and then delineate the FSL algorithm aimed at coping with non-IID data.

\subsection{Problem Formulation}

Consider the following ML problem in which we train a model to minimize an empirical loss:
\begin{align} 
\textstyle \min_{x \in \mathbb{R}^d} \quad F(x) \triangleq \frac{1}{n}\sum_{i\in [n]} \ell(x, s_i), 
    \label{eqProblem_Central}
\end{align} 
where $x \in \mathbb{R}^d$ is the vector of model parameters that need to be learned, $\mathcal{D} = \{s_1,\ldots,s_n\}$ is the set of training samples, and $\ell(x, s_i)$ is the loss for sample $s_i$ under model $x$. 

In FL, the goal remains the same, which is to minimize the total loss, but training data are distributed at multiple clients. 
Suppose that there are $N$ clients and 
the dataset $\mathcal{D}$ is partitioned into $\{\mathcal{D}_1, \mathcal{D}_2, \ldots, \mathcal{D}_N\}$, where $\mathcal{D}_i$ is the local dataset at client $i$. For each $i \in [N]$, define $n_i := |\mathcal{D}_i|$ and $f_i(x) := \frac{1}{n_i} \sum_{s\in \mathcal{D}_i} \ell(x, s)$ to be the loss function of client $i$ over its own dataset $\mathcal{D}_i$ under model $x$. Then, problem~\eqref{eqProblem_Central} can be reformulated as follows with $p_i = \frac{n_i}{n}$ for all $i\in [N]$: 
\begin{align} 
\textstyle \min_{x \in \mathbb{R}^d} \quad F(x) =\sum_{i\in [N]} p_i f_i(x).
    \label{eqProblem}
\end{align}

Suppose that the server also has access to a dataset $\mathcal{D}_0$ with $n_0 = | \mathcal{D}_0 |$. In the algorithm of \cite{Zhao2018}, a subset of samples in $\mathcal{D}_0$ is shared with {\em all} clients and is not utilized by the server. Each client $i$ implements the conventional FL algorithm using the augmented dataset $\mathcal{D}'_i = \mathcal{D}_i \cup \mathcal{D}_0$.\footnote{For simplicity, we either assume that $\mathcal{D}_i \cap \mathcal{D}_0 = \varnothing$ or consider any dataset as a multiset, allowing for possible multiple instances for each of its elements. Thus, we can write $|\mathcal{D}'_i| = |\mathcal{D}_i| + |\mathcal{D}_0|$.} Under such data sharing, the optimization problem in \eqref{eqProblem_Central} is modified as follows to reflect the change in clients' datasets: 
\begin{align*} 
\min_{x \in \mathbb{R}^d} ~~ F'(x) =\frac{1}{n+Nn_0}\Big(\sum_{s \in \mathcal{D}} \ell(x, s) + N\!\!\sum_{s' \in \mathcal{D}_0} \ell(x, s') \Big)
\end{align*} 
Similar to \eqref{eqProblem}, this problem can be rewritten using the weighted sum of clients' loss functions as follows:
\begin{align} 
\textstyle	\min_{x \in \mathbb{R}^d} \quad F'(x) = \sum_{i\in [N]} p_i' f'_i(x), \label{eqProblem2}
	\end{align} 
where $f'_i(x) = \frac{1}{n_i+n_0}\sum_{s\in \mathcal{D}'_i} \ell(x, s)$ is the modified loss of client $i$, 
and $p_i' = \frac{n_i+n_0}{n+Nn_0}$ is the corresponding weight. 

Define $f_0 := n_0^{-1} \sum_{s \in \mathcal{D}_0} \ell(x, s)$ to be the loss function for the samples in $\mathcal{D}_0$. Using the definition of $F$ in \eqref{eqProblem_Central}, the new objective function $F'$ can be rewritten as
\begin{align}
    F' 
    = \textstyle\frac{n}{n+Nn_0}\Big( F + \frac{Nn_0}{n} f_0 \Big) . \label{eq_defn_Ftilde}
\end{align}
This tells us that the above data sharing method alters the objective function by adding the loss function $f_0$ for the shared samples with a weight of $\frac{N n_0}{n}$. It also suggests that the quality of the solution obtained from \eqref{eqProblem2}, relative to the original problem in \eqref{eqProblem}, depends on how similar $F$ and $f_0$ are: when $F = f_0$, the two problems become equivalent. More importantly, it shows that sharing the samples in $\mathcal{D}_0$ with clients may be unnecessary; instead, the server can learn from $\mathcal{D}_0$ and combine its learned model with clients' models in a federated fashion. Having the server learn, rather than sharing training samples among the clients, avoids practical issues such as extra communication overheads, long and unpredictable network delays, and privacy concerns. It also allows us to choose the weight for $f_0$, which we denote by $\gamma$, to be different from $\frac{N n_0}{n}$, based on the quality of $\mathcal{D}_0$. This leads to a following (centralized) optimization problem:
\begin{align}
\textstyle \min_{x\in \mathbb{R}^d} \quad F + \gamma f_0.
    \label{eq:objective1}
\end{align}

Note that our problem formulation above can be generalized to the case with expected losses as follows:
\begin{align}
\textstyle \min_{x\in \mathbb{R}^d} \quad \Big( \sum_{i\in [N]} p_i f_i(x) \Big) + \gamma f_0(x), 
    \label{eq:objective2}
\end{align}
where $p = (p_i; i \in [N])$ is a probability vector, and $f_i(x) = \Ept_{z \sim \mathcal{D}_i}[f_i(x;z)]$ is the expected loss function of the server ($i=0$) and each client $i\in [N]$, and $\mathcal{D}_i$ is the corresponding data distribution. In what follows, we will use \eqref{eq:objective1} to facilitate our discussion and emphasize that our analysis applies directly to \eqref{eq:objective2}.





\subsection{FSL Algorithm} \label{subsec_main_algorithm}


We assume that the server has access to dataset $\mathcal{D}_0$ and will augment FL with what the server learns over $\mathcal{D}_0$. As stated earlier, we refer to this approach as {\em Federated Learning with Server Learning} or FSL. 


There are several ways to incorporate server learning (SL) into FL. One is to treat the server as a regular client that participates in every round of FL process \cite{Yoshida2020}: During each global round, the server updates the current global model using $\mathcal{D}_0$ and then aggregates it with the updated models reported by the clients.  
We call this approach {\em non-incremental} SL. 
One issue with non-incremental SL is that the weight for the server would be very small when $n_0 \ll n$, which means that the server's contributions, based on its learning from $\mathcal{D}_0$, to the global model will be minor. Moreover, this approach fails to exploit the good quality of $\mathcal{D}_0$, especially when its distribution is close to that of $\mathcal{D}$. This issue can be partially alleviated by increasing the weight given to the server's model in the aggregation step. 

These observations motivate us to consider an {\em incremental learning} scheme in which the server performs additional learning over dataset $\mathcal{D}_0$ based on the aggregated model, as shown in  Algorithm~\ref{FLSL} below in more detail. 
In particular, lines 1--9 of Algorithm~\ref{FLSL} are the same as in a conventional FL algorithm \cite{McMahan2017}, where in each global round $t$, each selected client $i$ (1) receives the current global model $x_t$ from the server, (2) performs $K$ steps of the Stochastic Gradient Descent (SGD) algorithm using its local data $\mathcal{D}_i$ (\textsc{LocalSGD}) with learning rate $\eta_l$, and (3) returns to the server its update $\Delta_t^{(i)}$. The server then combines its current model $x_t$ with the updates from the clients using some weight $\eta_g >0$ (lines~8--9). It then uses the resulting updated model  to learn locally by performing $K_0$ steps of \textsc{LocalSGD} with learning rate $\gamma\eta_0$ (line~10). As one can see, our approach has the same  computation and communication costs at the clients as the usual FL framework.

Note that our algorithm is similar to the incremental (stochastic) gradient method, which has been shown to be much faster than the non-incremental gradient method when the model is far from a (locally) optimal point  \cite{bertsekas2011incremental}. While FL with local SGD also works in an incremental fashion, it often needs small learning rates, hence longer learning times, to ensure convergence when the distribution of clients' data is heterogeneous. 

Before presenting a formal analysis and experimental results, let us provide some insights into FSL. First, when the distributions of $\mathcal{D}_0$ and $\mathcal{D}$ are close, server's loss function $f_0$ will be similar to the overall loss function $F$ in \eqref{eqProblem_Central}. Consequently, if the current model is far from an optimal point, the gradient $\nabla f_0$ will track the global gradient $\nabla F$, even when individual clients' gradients $\nabla f_i$ do not follow $\nabla F$ closely. Therefore, when the updated model obtained by aggregating clients' updated models does not make (much) progress, $\nabla f_0$ will help improve the updated model. In fact, it turns out that significant improvements can still be achieved even when the distributions of $\mathcal{D}_0$ and $\mathcal{D}$ are not very similar as long as their difference is small in relation to the non-IIDness of clients' data.  
We will elaborate on these points in the following section. 

\begin{algorithm2e}
\caption{\textsc{FSL}: FL with Server Learning} \label{FLSL}
\DontPrintSemicolon
\nonl\textbf{Server:}\; 
initialize $x_0$, $K$, $K_0$, $\eta_l, \eta_g, \gamma\eta_0$\;
\For{$t = 0, \ldots, T-1$}{
    sample a subset $\mathcal{S}$ of clients (with $|\mathcal{S}| = S$)\;
    broadcast $x_t$ to clients in $\mathcal{S}$\;
    \ForAll{$\mathrm{clients}~i \in \mathcal{S}$}{
        $x^{(i)}_{t,K} = \textsc{LocalSGD}(x_t, \eta_l, K, \mathcal{D}_i)$\;
        upload to server: $\Delta^{(i)}_t = x^{(i)}_{t,K} - x_t$\;
    }
    $\Delta_t = \frac{\sum_{i\in \mathcal{S}}\Delta^{(i)}_t}{|\mathcal{S}|}$\;
    $\bar{x}_t = x_t + \eta_g \Delta_t$\;
    $x_{t+1} = \textsc{LocalSGD}( \bar{x}_t, \gamma\eta_0, K_0,  \mathcal{D}_0)$\;
}
\nonl\;
\nonl$\textsc{LocalSGD}(x, \eta, K, \mathcal{D}_i)$:\; 
    $y_0 = x$\;
    \For{$k=0,\ldots,K-1$}{
        compute an unbiased estimate $g(y_k)$ of $\nabla f_i(y_k)$\;
        $y_{k+1} = y_{k} - \eta g(y_k)$\;
    }
\end{algorithm2e}

\section{Convergence Results}
    \label{sec:Convergence}

We first show in subsection~\ref{subsec_SL_correction} that SL can be viewed as a correction step for FL in the case of non-IID training data. Then the main convergence results are provided in subsection~\ref{subsec_convergence_result}. 

\subsection{SL as Corrections to FL When Far from Convergence} 
    \label{subsec_SL_correction}

In order to simplify our discussion presented in this subsection which provides key intuition behind our approach, assume that the server can compute gradient $\nabla f_0(x)$, $x \in \mathbb{R}^d$, and consider the usual gradient descent (GD) method for SL. 
First, consider a single update carried out by the server using GD, starting with some model $w_0$, i.e., 
$w_{1} = w_0 - \eta_0\nabla f_0(w_0)$.
Suppose that $\nabla F$ is Lipschitz continuous with parameter $L$.\footnote{This assumption is standard in FL and often holds when training neural networks. We will state this assumption formally in Section~\ref{subsec_convergence_result}.} Then,
\begin{align}
    F(w_1) -F(w_0) &\le 
     \inprod{\nabla F (w_0)}{w_1-w_0} + 0.5 L \|w_{1}-w_0\|^2 \nonumber \\
    &\le  -\eta_0\inprod{\nabla F (w_0)}{\nabla f_0(w_0)} +  0.5 L \eta_0^2 \|\nabla f_0(w_0)\|^2.
    \label{eq:ineq1}
\end{align}
The above inequality indicates that SL can improve FL further when the second term in \eqref{eq:ineq1} is sufficiently negative so that 
\\ \vspace{-0.22in}
\begin{align}
    2\inprod{\nabla F (w_0)}{\nabla f_0(w_0)} >  L\eta_0 \|\nabla f_0(w_0)\|^2. \label{grad_f0_cond_angle}
\end{align}
This condition holds when $\nabla f_0(w_0)$ makes an acute angle with $\nabla F(w_0)$ (provided that $\|\nabla F(w_0)\| >0$), in which case progress can be made by using a sufficiently small step size $\eta_0$. This will likely be the case when the distribution of $\mathcal{D}_0$ is similar to that of $\mathcal{D}$ and, when $w_0$ is far from a (local) minimizer, $-\nabla f_0(w_0)$ will likely be a descent direction of $F$ at $w_0$. 

In order to further see the role of $\mathcal{D}_0$, let us rewrite condition \eqref{grad_f0_cond_angle} as follows: 
\begin{align}
    &\|\nabla F (w_0)\|^2 + (1-L\eta_0)\|\nabla f_0 (w_0)\|^2 > \|\nabla F (w_0) -\nabla f_0 (w_0)\|^2. \label{grad_f0_cond_norm}
\end{align}
This implies the following. First, the error $\|\nabla F (w_0) -\nabla f_0 (w_0)\|^2$ in general depends on relationship between the server's dataset $\mathcal{D}_0$ and the aggregate dataset $\mathcal{D}$; the more dissimilar $\mathcal{D}_0$ is to $\mathcal{D}$, the larger the error and thus the smaller the improvement. In fact, SL can have negative impact if the error is sufficiently large. This suggests that the server dataset should be selected carefully in order to maximize the benefits of SL. As an example, consider $\mathcal{D}_0$ consisting of IID samples. In this case, the error $\|\nabla F (w_0) -\nabla f_0 (w_0)\|^2$ tends to decrease with the size of $\mathcal{D}_0$ according to (sampling without replacement)
\begin{align}
\Ept_{\mathcal{D}_0} \| \nabla f_0 (x) - \nabla F(x)\|^2 = 
\Big( \frac{n}{n_0} - 1 \Big)\frac{\tilde\sigma_0^2(x)}{n-1}, \label{eqServer_iid_variance}
\end{align}
where $\tilde\sigma_0^2(x) = \frac{1}{n}\sum_{s\in \mathcal{D}}\|\nabla_x \ell(x,s) - \nabla F(x)\|^2$ is the population variance. 
Thus, condition \eqref{grad_f0_cond_norm} can be satisfied by increasing $n_0$.

Second, for fixed $\mathcal{D}_0$ (of reasonable quality), the inequality in \eqref{grad_f0_cond_norm} holds when $\|\nabla F (w_0)\|$ is large, i.e., $w_0$ is far from being a stationary point, which is expected at the beginning of the training process. This is true even when $\nabla f_0(x)$ is a biased estimate of $\nabla F(x)$ as long as $\|\nabla F (w_0) -\nabla f_0 (w_0)\|^2$ is strictly smaller than $\|\nabla F (w_0)\|^2+\|\nabla f_0 (w_0)\|^2$, i.e., the angle between the gradients is acute as mentioned earlier, for a sufficiently small step size $\eta_0$.
Third, when $\|\nabla f_0 (w_0)\|$ is sufficiently small, e.g., when overfitting happens at the server, the improvement by SL is also insignificant. Finally, when $w_0$ is near a stationary point of $F$ but far from that of $f_0$, i.e., $\|\nabla F (w_0)\| \ll \|\nabla f_0(w_0)\|$, the inequality in \eqref{grad_f0_cond_norm} may be reversed, in which case SL can impair FL, pushing the model toward server's local stationary points. In this case, our algorithm does not yield exact convergence but oscillates between stationary points of $F$ and $f_0$, which is expected for an incremental gradient method \cite{bertsekas2011incremental}. Such convergence will be analyzed in details in the next subsection. 

The above analysis also applies when the server performs multiple updates. In particular, suppose that the server performs $K_0$ updates of the model using the GD method with a fixed step size $\eta_0$: 
\\ \vspace{-0.22in}
\begin{align*}
    w_{t,k+1} = w_{t,k} - \eta_0\nabla f_0(w_{t,k}), \quad k=0,\ldots,K_0-1,
\end{align*}
with $w_{t,0} = \bar{x}_t$ and $x_{t+1} = w_{t,K_0}$. 
Then, repeating the steps above and summing over the iterations, we obtain
\begin{align*}
    F(x_{t+1}) - F(\bar{x}_t) 
    &\le  - 0.5\eta_0 \textstyle \sum_{k=0}^{K_0-1} \big( \|\nabla F (w_{t,k})\|^2 + (1-L\eta_0)\|\nabla f_0 (w_{t,k})\|^2  \big) \\
    &\qquad\qquad +  0.5\eta_0 \textstyle \sum_{k=0}^{K_0-1}\|\nabla F (w_{t,k}) -\nabla f_0 (w_{t,k})\|^2.
\end{align*}
Similarly to the single-update case discussed earlier, we can see that carrying out multiple updates at the server is beneficial when $w_{t,k}$ is far from being a stationary point of either $F$ or $f_0$, more precisely, $\|\nabla F (w_{t,k})\|^2 + (1-L\eta_0)\|\nabla f_0 (w_{t,k})\|^2 > \|\nabla F (w_{t,k}) -\nabla f_0 (w_{t,k})\|^2.$ 
This also suggests that when learning collaboratively with clients, the server should not overfit its own data, which could happen easily when $n_0$  is small. 

\subsection{Convergence Analysis} 
    \label{subsec_convergence_result}

In this subsection, we study the convergence properties of FSL. Specifically, we will prove that, under suitable conditions on step sizes, FSL converges to a neighborhood of a stationary point of the following modified loss function 
$$\tilde{F} = \textstyle\frac{1}{1+\gamma} F + \frac{\gamma}{1+\gamma}f_0$$ 
which is simply the normalized version of that in \eqref{eq:objective1}, where the weight $\gamma>0$ is chosen by the server. 
%
The value of $\gamma$ should depend on the quality of server's dataset $\mathcal{D}_0$: when the distribution of $\mathcal{D}_0$ is close to that of $\mathcal{D}$, a larger value would offer greater benefits. But, our analysis presented below does not assume that their distributions are close. Also, our experimental results demonstrate that the FSL algorithm can deliver significant benefits even when the two distributions differ considerably (see Section~\ref{sec_experiment})

First, we state several assumptions under which our analysis of Algorithm~1 is carried out.

\begin{assumption} \label{assm_LossFunc}
The server and client's local loss functions $\{f_i\}_{i = 0}^N$ are $L$-smooth on $\mathbb{R}^d$.
\end{assumption}

This assumption is standard in the literature and is often satisfied in practice. It also implies that 
the global loss functions $F$ and $\tilde{F}$ are $L$-smooth. 
The second assumption is used to bound the gradient dissimilarity caused by clients' non-IID data; see, e.g., \cite{reddi2020adaptive}.

\begin{assumption}\label{assm_globalGrad}
There exists a finite constant $G$ such that $\frac{1}{N}\sum_{i \in [N]} \| \nabla f_i(x) -\nabla F (x) \|^2 \le G^2$ for all $x\in \mathbb{R}^d$.
\end{assumption}



Here, $G$ bounds the average disparity between the gradients of clients' loss functions and the empirical loss caused by non-IID samples at the clients; the IID case corresponds to $G \to 0$. 
Similarly, when the distributions of $\mathcal{D}_0$ and $\mathcal{D}$ are different, there can be a discrepancy between $\nabla f_0$ and $\nabla F$. We use the following assumption to characterize the quality of server dataset $\mathcal{D}_0$.


\begin{assumption} \label{assm_server_data}
There exists a finite constant $\bar{\xi}$ such that $\| \nabla f_0(x) - \nabla F(x) \|^2   \le \bar{\xi}^2$ for all $x\in \mathbb{R}^d$.
\end{assumption}

\blue{ 
This assumption does {\bf not} imply that the server data distribution is similar to that of the
clients' aggregate data (although this would be an ideal situation). In other words, $\bar{\xi}^2$ is not necessarily small, and \underline{our analysis presented below} \underline{examines how this bound affects the performance of FSL}.}

Note that the uniform bounds in Assumptions~\ref{assm_globalGrad} and \ref{assm_server_data} are introduced to simplify presentation; what we need in our analysis is that the bounds hold for the sequence $\{x_t\}_{t\ge 0}$ generated by our algorithm. This holds, for example, when $\{x_t\}$ is bounded. \blue{Although those bounds are usually unknown, they quantify the extent of non-IIDness in clients' and server's data and facilitate our analysis.} 

Finally, we assume that the clients and the server can obtain unbiased noisy estimates of the gradient of their local loss functions for updating their local models.This assumption is also standard in stochastic optimization.

\begin{assumption}\label{assm_localGrad}
All clients $i \in [N]$ and the server ($i =0$) have access to unbiased estimates $g_i$ of $\nabla f_i$ with variance bounded by $\sigma_i^2$. For simplicity, we further assume that $\sigma_i = \sigma$ for all $i \in [N]$. 
\end{assumption}

Here, $\sigma_i$ only bounds the variance of noisy estimates for the clients and the server. Note that it is not uncommon in practice that the server has enough computing capability so that it can obtain gradient estimates with small variance. For example, when $n_0$ is sufficiently small, the server may utilize all samples to compute the exact gradient for each update, in which case we have $\sigma_0= 0$. 

Let us now briefly describe the idea to prove the convergence of FSL. For the special case when $N=1$, $K=1$, and $\sigma_i=0$, FSL simply reduces to the incremental gradient method. For a general case, we can relate the sequence $\{x_t\}$ generated by our algorithm to that of a centralized incremental stochastic gradient method applied to the global loss function $\tilde{F}$, where the difference between the two is caused by client sampling and local learning steps. As a result, by choosing step sizes sufficiently small in connection with the bounds in Assumptions~\ref{assm_LossFunc}--\ref{assm_localGrad}, we can bound such differences and relate the convergence of the two algorithms. 

Our first result below demonstrates the progress in one global round of FSL, which resembles that of a centralized stochastic gradient algorithm. Here, we use $\Ept_t[\cdot]$ to denote the conditional expectation\footnote{This conditional expectation is given the $\sigma$-algebra generated by random variables that determine $x_t$.} over the randomness at round $t$ and define the following: 
$$\textstyle\rho_s = \frac{N-S}{N-1},\quad \Psi = \frac{\gamma^2\sigma_0^2}{K_0} + \frac{\sigma^2}{KS} + \frac{\rho_s G^2}{S}.$$

\begin{theorem} \label{thm_descent_lem}
Suppose that Assumptions \ref{assm_LossFunc}--\ref{assm_localGrad} hold, and let the step sizes satisfy 
\begin{align}
    K \eta_l \eta_g = K_0\eta_0 \le \textstyle \frac{1}{4L}\min\{\eta_g, 1/\gamma, 8/9\} . \label{eq_stepsize_mainCond}
\end{align}
Then, the following holds for any $t\ge 0$: 
\begin{align}\label{eq_descent_lemma}
    \Ept_{t} [\tilde{F}({x}_{t+1})] 
    \le \tilde{F}(x_t) - K_0\eta_0 h\|\nabla \tilde F(x_t)\|^2  + 5K_0^2\eta_0^2 L\Psi + 8K_0^3\eta_0^3L^2\big( \textstyle\frac{\gamma\kappa}{1+\gamma} \bar{\xi^2} + \Phi \big)
\end{align}
where $h = \gamma + \frac{1}{2} - K_0\eta_0 L\frac{1+\gamma}{2}\big( 3\gamma+3 + 16\kappa K_0\eta_0L \big)$, $\kappa = \max\{4\gamma^3, 2\eta_g^{-2} + 3\gamma^2\}$ and 
$\Phi = \gamma^2\Psi + \frac{2\gamma^2}{S}\big( \frac{\sigma^2}{K} + \rho_s G^2 \big) + \eta_g^{-2}\big( 2G^2+\frac{\sigma^2}{K}\big)$.
\end{theorem}
\begin{IEEEproof}
See Appendix~\ref{proof_descent_FL}.
\end{IEEEproof}

We have the following remarks. 
First, condition \eqref{eq_stepsize_mainCond} means that the server and the clients use the same effective step size per round, which is sufficiently small in the order of $\mathcal{O}(1/L(1+\gamma))$. Second, by choosing a sufficiently small $K_0\eta_0$, we have $h \ge 1/2$; in fact, it can be shown that if 
\begin{equation}
    K_0\eta_0 \le  \textstyle\frac{1}{8L(\gamma+1)}\min\big\{ 1,\frac{(\gamma+1)^2}{2\kappa} \big\},  \label{eq_stepsize_mainCond2}
\end{equation}
then $h \ge \frac{3\gamma+1}{4}$. Thus, when the current model is far from a stationary point and $\|\nabla \tilde F(x_t)\|^2$ is large, it is desirable to use large $\gamma$. But, if $\gamma$ is too large, the last two terms in \eqref{eq_descent_lemma} will likely dominate and prevent the algorithm from making significant improvements, potentially causing it to diverge. Although this suggests that one could use a diminishing~$\gamma$, we consider a fixed $\gamma$ in our analysis for simplicity.

Using the result above, we can quantify the overall progress of the algorithm as follows. 
\begin{theorem}\label{thm_progress_overall}
Suppose Assumptions \ref{assm_LossFunc}--\ref{assm_localGrad} and condition \eqref{eq_stepsize_mainCond} hold. Let $\mathcal{E}_T \!=\! \min_{t \le T-1} \Ept \|\nabla \tilde F(x_t)\|^2$. Then,
\begin{align*}
    h\mathcal{E}_T 
    \le \textstyle \frac{\tilde{D}_0}{TK_0\eta_0}   + 5K_0\eta_0 L\Psi  + 8K_0^2\eta_0^2L^2\big( \textstyle\frac{\gamma\kappa}{1+\gamma} \bar{\xi}^2 + \Phi \big)
\end{align*}
for any $T > 0$, where $\tilde{D}_0 =\tilde{F}(x_0) - \tilde{F}^*$. 
\end{theorem}

Note that 
$\Phi = \frac{\gamma^4\sigma_0^2}{K_0} + \phi$ with $\phi =  \frac{3\gamma^2}{S}\big( \rho_s G^2 + \frac{\sigma^2}{K} \big) + \frac{1}{\eta_g^2} \big( 2G^2+\frac{\sigma^2}{K}\big)$. 
When $S \ll N$, we have $\rho_s  \approx 1$, and $\phi = \Theta \big( ( \frac{\gamma^2}{S} + \frac{1}{\eta_g^2}) ( G^2+\frac{\sigma^2}{K} ) \big)$. 
As both $\phi$ and $\kappa$ decrease in $\eta_g$, in principle we can select large $\eta_g$ to reduce the upper bound in Theorem~\ref{thm_progress_overall}. Here, since we are interested in scenarios where $\gamma = \mathcal{O}(1)$ and $\rho_s \approx 1$,  
$\eta_g$ need not be too large either. Based on these observations, let us consider $\eta_g = \Theta(\sqrt{S})$, which gives $\phi = \Theta\big(\frac{\gamma^2+1}{S}( G^2+\frac{\sigma^2}{K}) \big)$ and thus 
$\Phi = \mathcal{O}\big((\gamma^2+1) \Psi \big)$. Under these conditions, we have the following result.

\begin{corollary}   \label{coro:1}
If $\eta_g = \Theta(\sqrt{S})$, $K_0 = \Theta(K)$,  $K_0\eta_0 = \Theta \big( \frac{\sqrt{KS}}{\sqrt{LT}(\gamma+1)} \big)$, and condition \eqref{eq_stepsize_mainCond2} hold, then 
    \begin{align}\label{eqConvergenceError}
        \mathcal{E}_T \!=\! \mathcal{O} \textstyle \left( \frac{\sqrt{L}}{\sqrt{KST}}   \big( \tilde{D}_0 \!+\! \frac{M^2}{(\gamma+1)^2} \big)  + \frac{L}{T} \big( \frac{M^2}{1+\gamma} \!+\! \frac{\gamma\kappa KS\bar{\xi}^2}{(1+\gamma)^4} \big) \right)
    \end{align}
    with $M^2 = \gamma^2\sigma_0^2S +\sigma^2 + \rho_s KG^2$.
\end{corollary}




Let us make the following remarks. First, the above sublinear rate of $\mathcal{O}(\frac{1}{\sqrt{T}})$ is to be expected for FL with a nonconvex loss function and is also similar to that of the usual SGD method. 

\blue{ 
Second, the FedAvg~\cite{McMahan2017} is a special case of FSL with $\gamma = 0$, i.e., without server learning. In this case, $M^2=\sigma^2 + \rho_s KG^2$ and thus $\mathcal{E}_T = \mathcal{O} \big(\frac{\sqrt{L}}{\sqrt{KST}}   \big( \tilde{D}_0+ M^2 \big)  + \frac{L}{T} M^2 \big)$ is large when clients' data is highly nonhomogemeous and $G^2$ is large. In this case, increasing $\gamma$ can alleviate the adverse effect of non-IID data, as the dependence on $G^2$ scales as $\mathcal{O}\big(\frac{G^2}{\sqrt{T}(1+\gamma)^2} + \frac{G^2}{T(1+\gamma)} \big)$, assuming that the last term in \eqref{eqConvergenceError} is not dominant. This happens when $\bar{\xi}^2$ is small compared to $G^2$ and $\gamma$ is not too large, 
especially in cases of our interest where $\sigma_0 \ll \sigma$ and $\bar{\xi}^2 \ll G^2$. We discuss two examples scenarios:
(1) The server samples are taken from $\mathcal{D}$ via uniform sampling without replacement\footnote{In this case, $\bar{\xi}$ tends to decrease with the size of $\mathcal{D}_0$ according to 
$\Ept_{\mathcal{D}_0} \| \nabla f_0 (x) - \nabla F(x)\|^2 = 
\big( \frac{n}{n_0} - 1 \big)\frac{\tilde\sigma_0^2(x)}{n-1}$, 
where $\tilde\sigma_0^2(x) = \frac{1}{n}\sum_{s\in \mathcal{D}}\|\nabla_x \ell(x,s) - \nabla F(x)\|^2$ is the population variance.} (2) In the applications we target, such as AVs, the
manufacturers can likely ensure that the samples collected by test vehicles are more diverse and representative (than those of
a typical client) as the collection process is under their control. Thus, it is likely that the server’s data would
be more representative than those of a typical “single” client and $\bar{\xi}^2$ is likely much smaller than $G^2$. We will experiment with these scenarios in the following section.}

Third, note that $\tilde{D}_0 =\frac{D_0 + \gamma(f_0(x_0) - f_0(\tilde{x}^*))}{1+\gamma}$, where $D_0 = F(x_0) - F(\tilde{x}^*)$ and $\tilde{x}^*$ is any global minimizer of $\tilde{F}$. If $x_0$ is chosen far from $\tilde{x}^*$ or a stationary point and the distributions of $\mathcal{D}_0$ and $\mathcal{D}$ are similar, it is likely that $\tilde{D}_0$ is large and $\tilde{D}_0  \approx D_0$. On the other hand, if the server pre-trains its model using its own data so as to minimize $f_0$, then $\tilde{D}_0$ can be improved. In fact, because of small size of the server dataset, overfitting can happen and thus $f_0(x_0) \approx 0$ and $\tilde{D}_0 \le \frac{D_0}{1+\gamma}$. This shows that both pre-training and increasing $\gamma$ help.

Forth, the first term of the bound in \eqref{eqConvergenceError} often dominates and scales as $\mathcal{O}(\frac{\tilde{D}_0 + \gamma^2\sigma_0^2S + \sigma^2}{\sqrt{KST}} + \frac{\rho_s\sqrt{K}G^2}{\sqrt{ST}})$. This implies that while increasing $K$ helps reduce the effect of stochastic noises and initialization, it increases client and server drifts and consequently amplifies the effect of non-IIDness (via the terms $\frac{\sqrt{K}G^2}{\sqrt{TS}}$ and $\frac{K \rho_s G^2}{T}+\frac{K\bar{\xi}^2}{T}$). Similarly, increasing $S$ will reduce the dominant term, which scales as  $\mathcal{O}(\frac{1}{\sqrt{S}})$, at the cost of slightly increasing the 
smaller term $\mathcal{O}(\frac{S\bar{\xi}^2}{T})$. 

Finally, let us remark on the optimality of the original loss. Since $\|\nabla F(x_t)\|^2  \le (1+\gamma)\|\nabla \tilde F(x_t)\|^2 +\frac{\gamma }{1+\gamma } \bar{\xi}^2$, it follows that $\min_{t \le T-1} \Ept \|\nabla F(x_t)\|^2 \le (1+\gamma)\mathcal{E}_T +\frac{\gamma }{1+\gamma } \bar{\xi}^2$. Here, $\mathcal{E}_T$ can be bounded using Corollary~\ref{coro:1}, while the second term affects the neighborhood to which the model converges. Thus, in principle, one should select $\gamma$ judiciously to trade off between these two terms. However, we show numerically in the next section that this can be done fairly easily. 


\section{Experimental Results}
    \label{sec_experiment}
    
    

We now illustrate the benefits of FSL  through experiments using two datasets CIFAR-10  \cite{Krizhevsky2009} and EMNIST \cite{cohen2017emnist}. 

\subsection{Setup}

\underline{\em Data and Model:} For CIFAR-10 and EMNIST, we use, respectively, 50k samples with 10 label classes and 108k samples with 45 label classes for training. Each dataset also has 10k samples for testing. For simplicity, we partition the $n$ training samples roughly evenly among $N$ clients so that each client~$i$ has $n_i \!=\! \frac{n}{N}$ samples of $C$ label classes. Each client will have $\frac{n_i}{C}$ samples per label class, selected uniformly at random without replacement from training data. 
We vary $C$ to study the effect of client data heterogeneity -- smaller $C$ represents more non-IID client data. We use neural networks with two convolutional layers and two dense hidden layers and cross-entropy as the loss function for training; see 
Appendix~\ref{subsec_NN_models}  
for further details.


\underline{\em Methods} We compare our approach \texttt{FSL} against (1) Federated Learning FedAvg (\texttt{FL}) \cite{McMahan2017}, (2) \texttt{FL} combined with Data Sharing (\texttt{DS}) \cite{Zhao2018} that requires {\em sharing among all clients a common dataset} comprising samples uniformly distributed over classes, (3) \texttt{FedDyn} \cite{durmus2021federated} which requires additional client storage to retain their state, and (4) SCAFFOLD \cite{Karimireddy2020} that {\em doubles  communication overheads} compared to other methods. 
For \texttt{FSL}, we assume that the server dataset $\mathcal{D}_0$ has $n_0$ samples taken from the training data. To facilitate comparison, we use $\mathcal{D}_0$ as the dataset shared among  clients in \texttt{DS}. We also tested \texttt{FSL} with non-incremental SL (\S~\ref{sec_prob}), but put its results in 
Appendix~\ref{subsec_compare_FSL_nonIncremental} 
for reference as it underperforms \texttt{FSL}.

\underline{\em Implementation}
Each client chosen by the server at each round trains its local model for $E_c$ epochs using local data with batch size $B$. In \texttt{FSL}, the server also updates its model 
for $E_s$ epochs in each round using batch size $B_0$. Here, we fix $E_c=1$, $E_s =\lceil \frac{n}{N n_0}E_c \rceil$, $B_0=B$, $\eta_g = \sqrt{S}$, and set $\eta_0 = \sqrt{S}\eta_l K/K_0$.

\underline{\em Evaluations} 
We run all algorithms for $T=1,000$ rounds, and compare their test accuracy (averaged using a rolling window of size 20) and convergence time measured by \textit{rise time}, which we define as the first time the test accuracy reaches $90\%$ of the final accuracy. The reported numbers are the averages of 3 runs. 

We consider the following two scenarios (which mimic the settings in \cite{Zhao2018} and \cite{Karimireddy2020}, respectively): (1) $N$ is small, and server data are IID and of small size compared to client's data, 
and (2) $N$ is large, server data are non-IID, and client data size is relatively small.

\subsection{Scenario 1}
    \label{subsec:NumericalResults}
Consider $(N,n_i,n_0)=(10, 5000, 500)$ for CIFAR-10 and $(45, 2400, 225)$ for EMNIST. Here, $\mathcal{D}_0$ has roughly $\frac{n_0}{C}$ samples per label class, sampled without replacement from  $\mathcal{D}$. We study the role of different parameters in \texttt{FSL} and compare it against \texttt{FL} and \texttt{DS}. Both \texttt{DS} and \texttt{FSL} use a pretrain step where the server trains its local model using SGD with learning rate of $0.01$ for $500$ epochs over its data $\mathcal{D}_0$ with batch size $B$. We varied $\gamma \in \{  \frac{Nn_0}{n}, 0.5, 1 , 1.5, 2 \}$; note that when $\gamma = \frac{Nn_0}{n}$, \texttt{FSL} has the same (global) objective as \texttt{DS}.

\underline{\em Effects of Client Data Distributions:}
Fig.~\ref{fig_emnist_varyC} shows the test accuracy as we vary $C$ to create different levels of non-IIDness.  We have the following observations. 

First, all algorithms achieve a similar final accuracy in the IID case ($C=45$ for EMNIST and $C=10$ for CIFAR-10). But, when client data become more non-IID as $C$ decreases, \texttt{FL} suffers significantly in both accuracy and convergence time, which is expected and well reported in the literature. 
Second, \texttt{DS} greatly improves over \texttt{FL}, but has a similar convergence property: slower learning with wide oscillations. This is to be expected as \texttt{DS} is essentially \texttt{FL} where each client has an additional small set of shared data. 
Third, in all cases, \texttt{FSL} provides the highest accuracy and fastest convergence with considerable acceleration at the beginning and much smaller oscillations in accuracy, thanks to only a small dataset at the server (which is about $0.21\%$ of training data for EMNIST and $1\%$ for CIFAR-10). 
Fourth, \texttt{FSL} performs fairly consistently for a range of $\gamma$ values, suggesting that fine tuning might be unnecessary. 
Finally, although we use a pretrained model for \texttt{FSL} and \texttt{DS} but not \texttt{FL}, we show in 
Appendix~\ref{subsec_wo_Pretrain} 
that similar observations can be obtained when \texttt{FSL}, \texttt{DS}, and \texttt{FL} all use the same initial model. In fact, FSL provides more significant acceleration, even in the IID cases, whereas \texttt{DS} offers little to no benefits over  \texttt{FL}.

\begin{figure}[tb]
\centering
\includegraphics[scale=0.55]{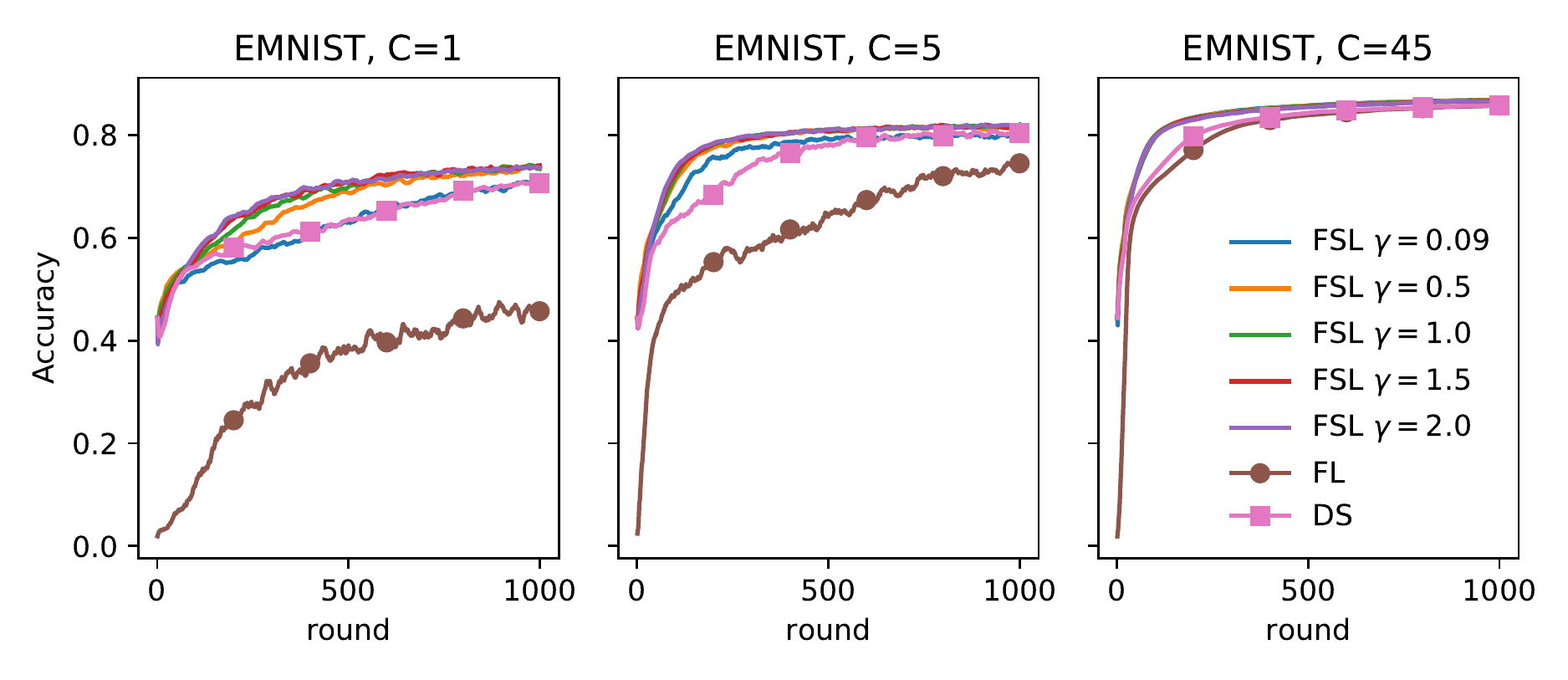}
    \includegraphics[scale=0.55]{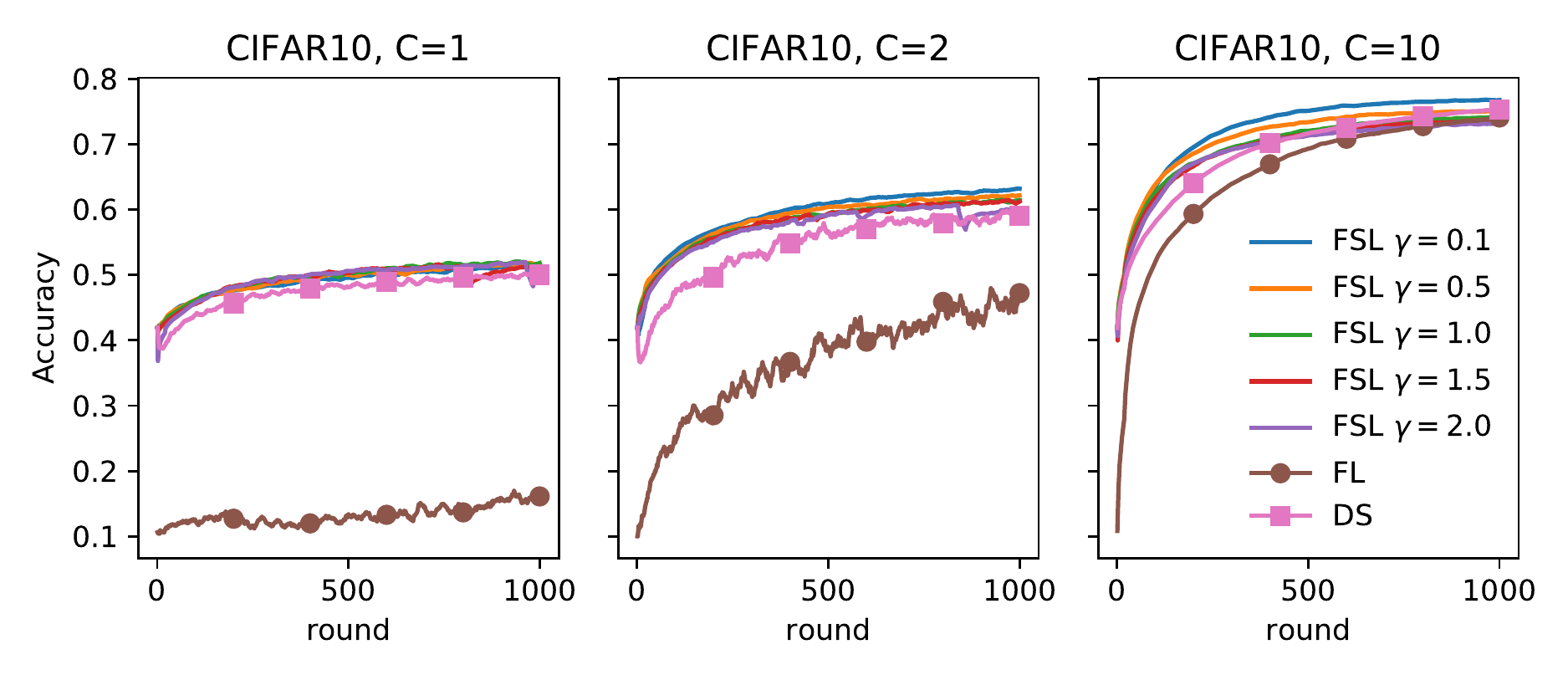}
    \vskip -0.1in
    \caption{Test accuracy for different values of $C$ and $\gamma$, where $n_0 = 225$, $S = 5$ and $\eta_l = 0.01$ for EMNIST, and $n_0 = 500$, $S = 4$ and $\eta_l = 0.01$ for CIFAR-10.}
    \label{fig_emnist_varyC}
    \vskip -0.1in
\end{figure}

\begin{figure*}[!t]
    \centering
    \includegraphics[scale=0.45]{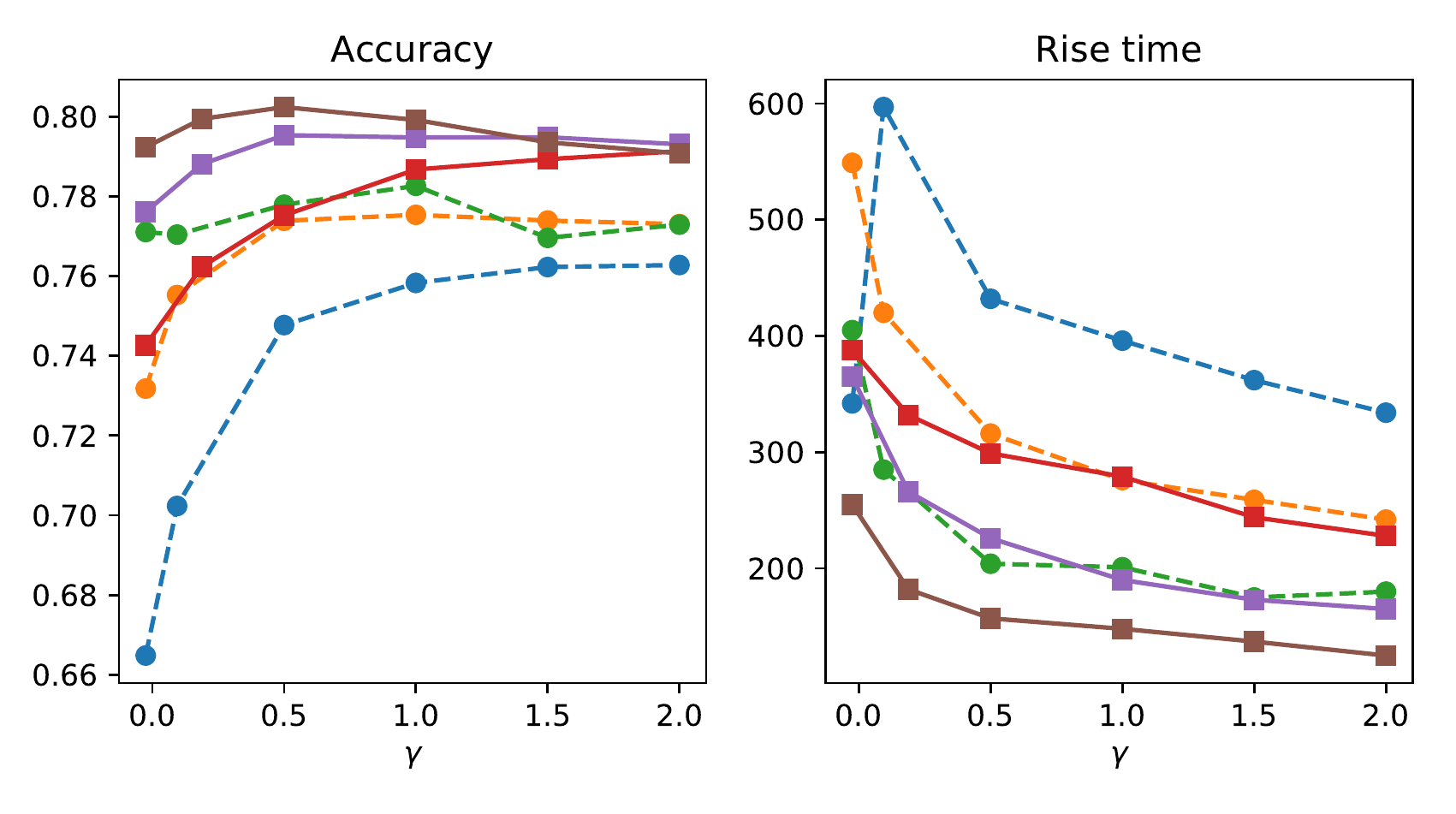}
    \hspace{2ex}
    \includegraphics[scale=0.45]{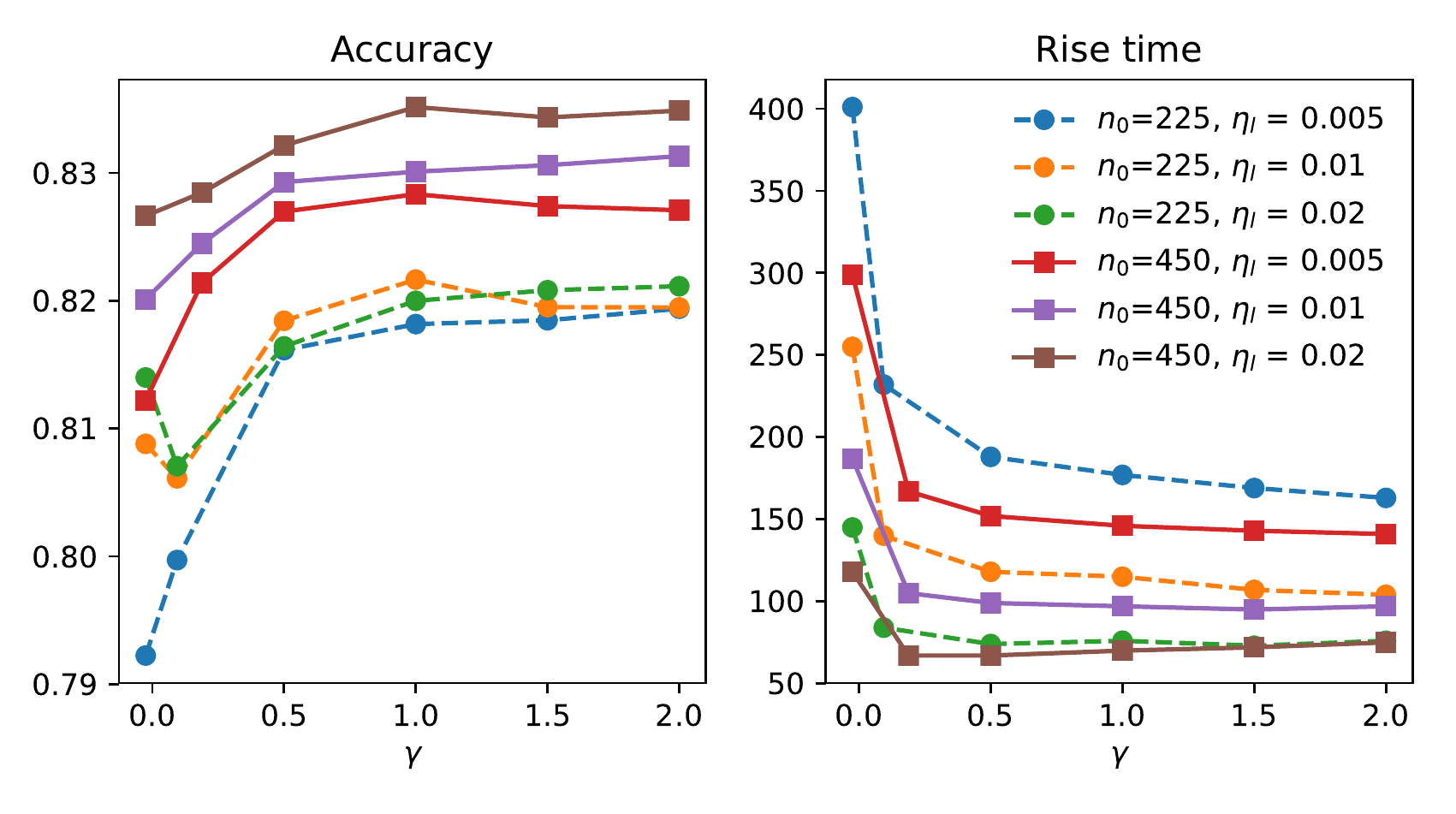}
    \vskip -0.1in
    \centerline{(a) EMNIST: $C\!=\!1, S\!=\!10$ \hspace{1.6in} (b) EMNIST: $C\!=\!5, S\!=\!5$}
    
    \includegraphics[scale=0.45]{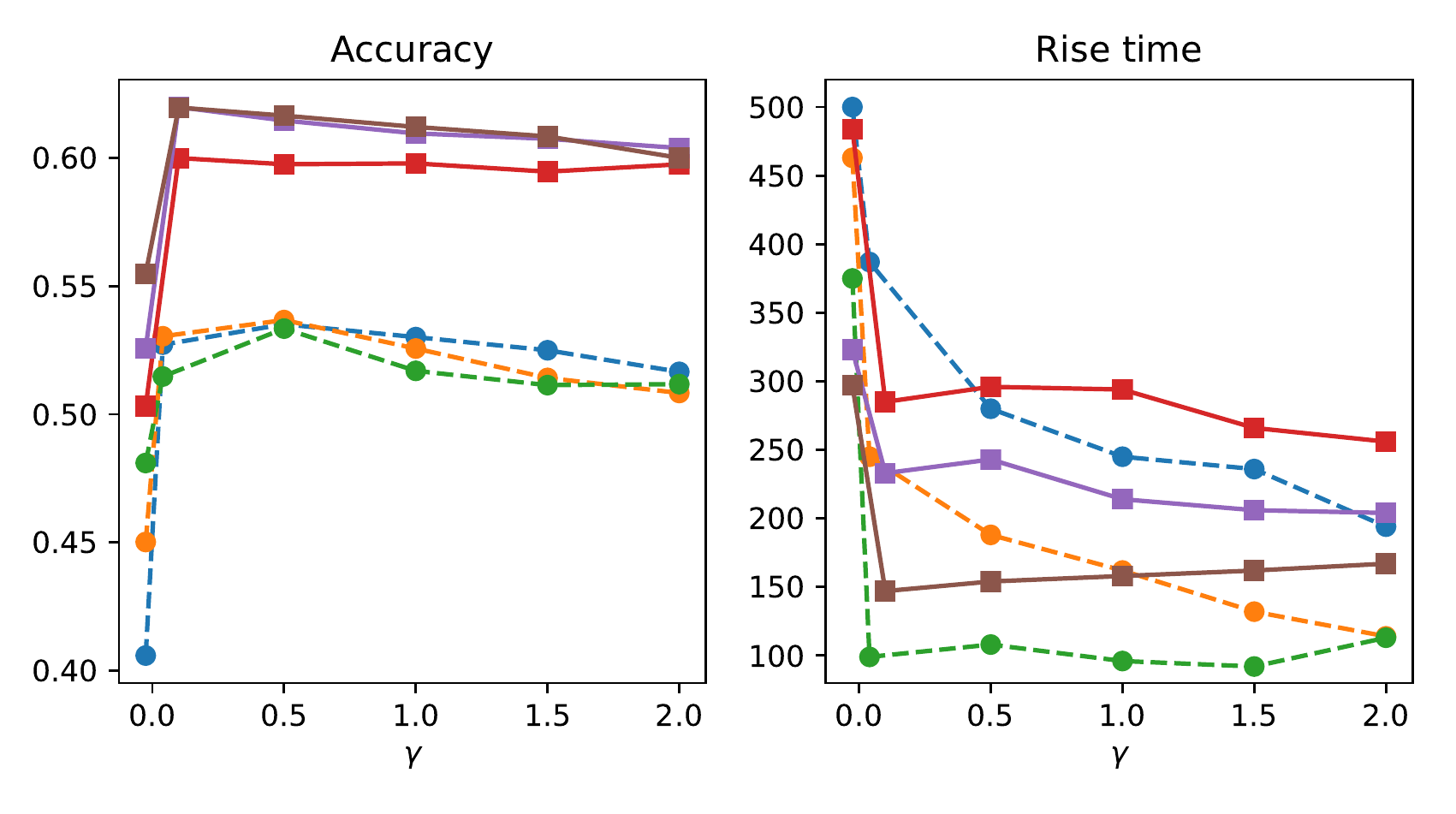}
    \hspace{2ex}
    \includegraphics[scale=0.45]{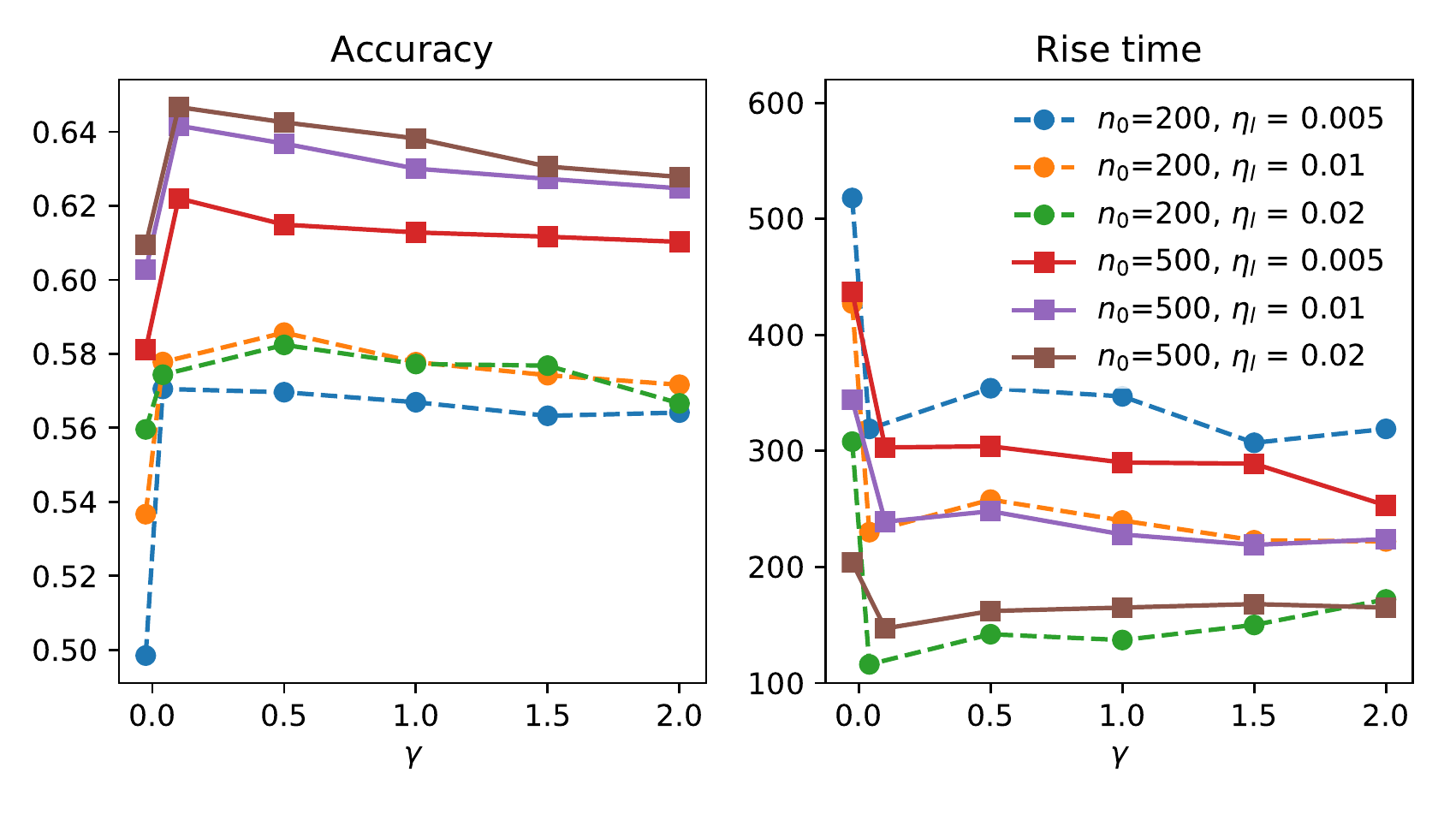}
    \vskip -0.1in
    \centerline{(c) CIFAR-10: $C\!=\!2, S\!=\!2$ \hspace{1.6in} (d) CIFAR-10: $C\!=\!2, S\!=\!4$}
    \caption{Test accuracy \& rise time vs. weight $\gamma$. Values at $\gamma \!>\! 0$ represent \texttt{FSL} and those at $\gamma \!=\! 0$ represent \texttt{DS} (instead of \texttt{FL} since \texttt{DS} outperforms \texttt{FL}).}
    \label{fig_acc_rise}
    \vskip -0.15in
\end{figure*}

\underline{\em Benefits of Server Learning:}   Fig.~\ref{fig_acc_rise} plots the accuracy and rise time of \texttt{FSL} when varying the weight $\gamma$, learning rate $\eta_l$, and server data size $n_0$. 

\noindent \quad  {\bf Role of $\gamma$:} First, in general, increasing $\gamma$ from $0$ improves the accuracy and convergence time significantly compared to \texttt{DS}. The improvement is more pronounced when comparing to \texttt{FL}. Second, such improvements remain significant over a wide range of $\gamma$ values. For example, $\gamma$ over $[0.5, 1.5]$ provides similar performance for all considered local learning rates $\eta_l$, server data sizes $n_0$, and for both datasets. For CIFAR-10, it appears that a smaller $\gamma$ provides better results, while a large value may slightly degrade the performance; the opposite holds true for EMNIST (except when $C=1$ and $\eta_l$ is large, increasing $\gamma>1$ actually decreases the accuracy). This can be attributed to the fact that the client data are more non-IID and server samples are more dissimilar in CIFAR-10 than in EMNIST; see the cases $C=2$ and $C=5$ in Fig.~\ref{fig_emnist_varyC}. 

\noindent \quad  {\bf Server data size:} First, with a small (good quality) dataset, the server can already have a pretrained model much better than random initialization. Second, it is clear that increasing the server data size helps improve \texttt{FSL} further. Here, the accuracy improvement is greater for CIFAR-10 than EMNIST. 
The rise time improvement is significant when $\eta_l$ is small and diminishes for larger $\eta_l$. Note that increasing the local learning rate $\eta_l$ also increases the server's effective learning rate, which is $\gamma\eta_0 =\gamma\eta_l \sqrt{S}$. 


    

\subsection{Scenario 2}
Consider $(N,n_i) =$ (1000, 50) for CIFAR-10 and (450, 240) for EMNIST. Unlike in Scenario 1, we now consider two different sources of data for the server.

\paragraph{Data from a few clients} The server obtains data only from a subset of $c$ clients,\footnote{These clients can be, for example, test vehicles in our AV example; here they are sampled without replacement once prior to training for simplicity.} each contributing $s$ samples (selected uniformly at random without replacement). Here $(c,s,n_0)=(10, 50, 500)$ for CIFAR-10 and $(9, 50, 450)$ for EMNIST. \blue{Note that the server data is imbalanced and non-IID (likely missing one or more label classes when $C = 1$).} 

\paragraph{Data from other source(s)}  For EMNIST, we provide the server $n_0 = 675$ \underline{synthetic} examples by generating for each label class 15 images of the corresponding letter or number using a cursive font with 5 rotation angles $\{-20, -10, 0, 10, 20\}$ and 3 sizes;\footnote{To generate synthetic data, we first plot each character or number in a 2 inch $\times$ 2 inch figure using  font sizes $\{100, 110, 120\}$ in points with each point equal to $1/72$ inch, and then resize it to a 28 pixel $\times$ 28 pixel figure.} see Fig.~\ref{fig_synthetic_emnist} for a comparison of this synthetic data and EMNIST. For CIFAR-10, we collect $n_0 = 504$ images from the dataset STL-10 with 9 similar label classes as in CIFAR-10, each with 56 examples;\footnote{STL-10 images were acquired from labeled examples on ImageNet; data available at: https://cs.stanford.edu/$\sim$acoates/stl10/} see Fig.~\ref{fig_synthetic_cifar} for an illustration of this data, and note that the class \texttt{frog} is absent in STL-10. We refer to our algorithm in this case as \texttt{FSLsyn}. Our goal with \texttt{FSLsyn} is to examine the benefits of server learning when it is performed on data with a significantly different distribution than that of clients' data.

\begin{figure}[!tb]
\centering
\includegraphics[width=0.26\linewidth]{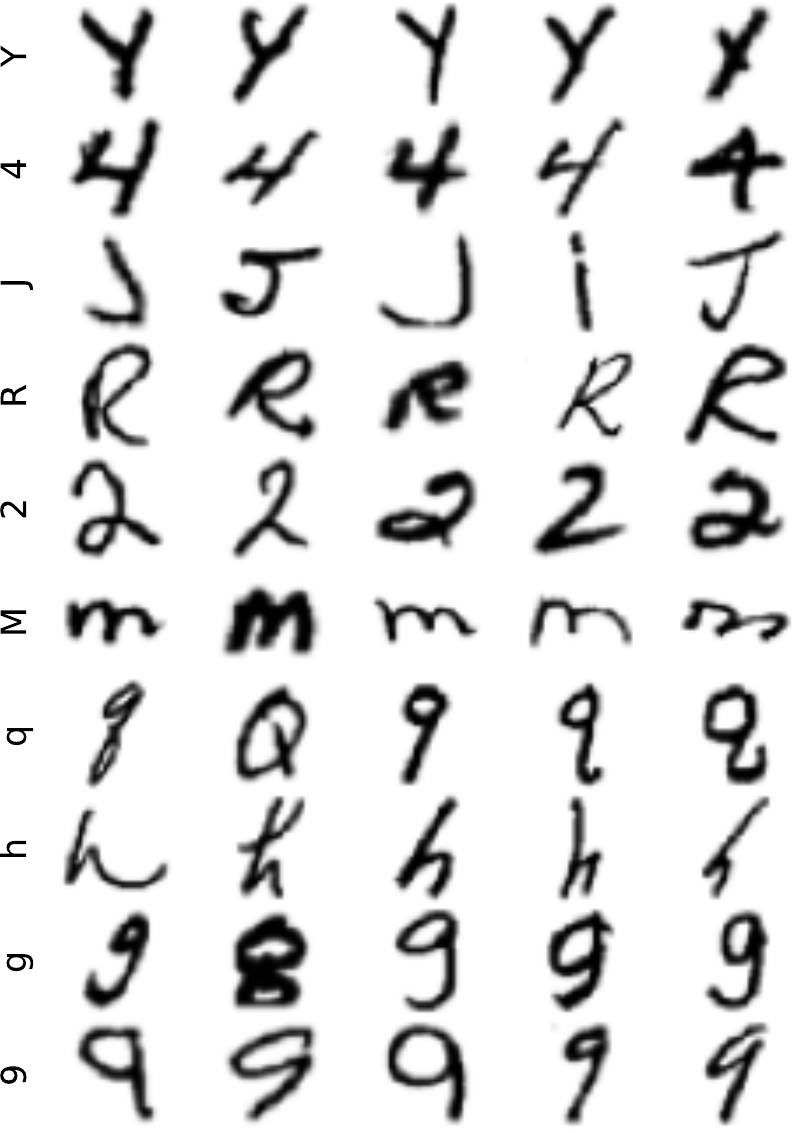}
\vline
\includegraphics[width=0.245\columnwidth]{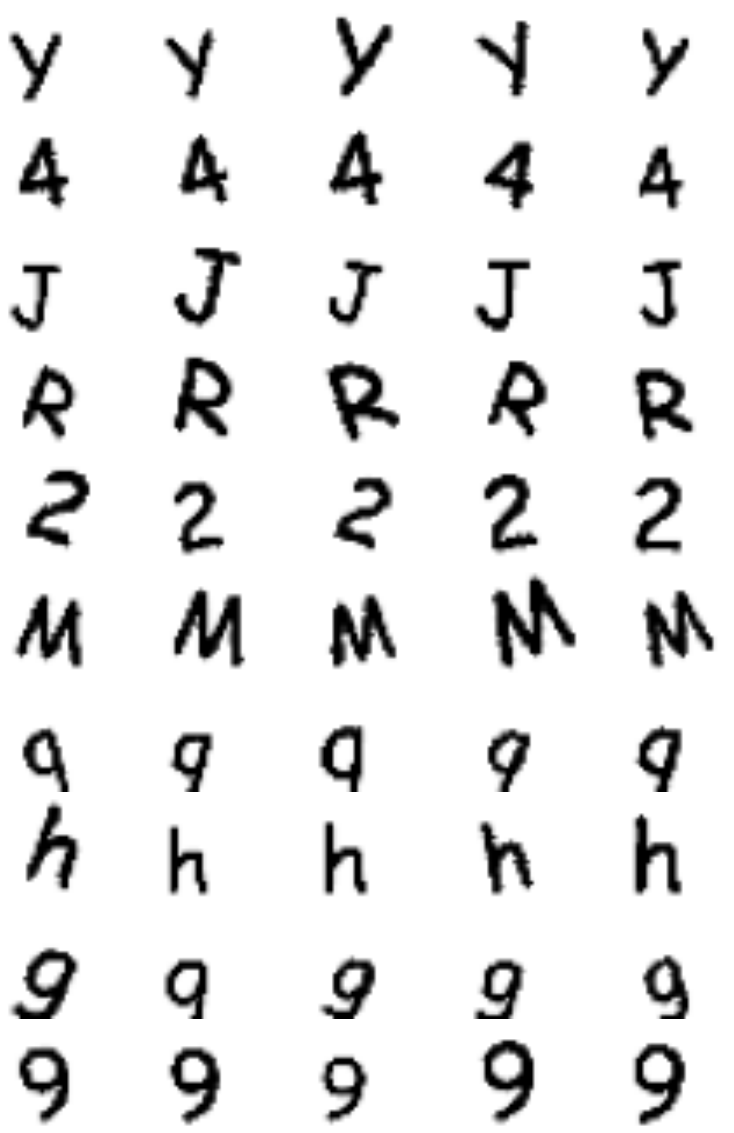} 
    \caption{Left: EMNIST training examples. Right: Server's synthetic examples.}
    \label{fig_synthetic_emnist}
\end{figure}
\begin{figure}[!tb]
\centering
\includegraphics[width=0.26\linewidth]{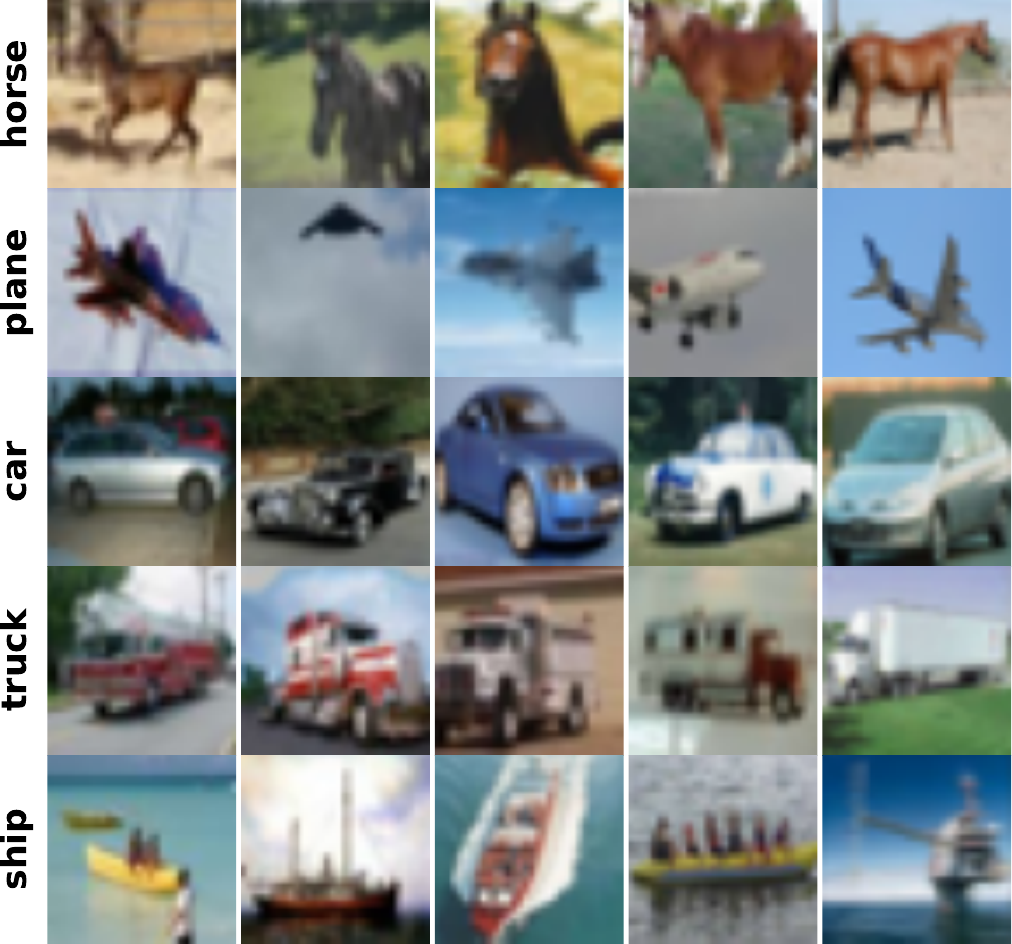}
\hspace{-0.025in} 
\includegraphics[width=0.2476\columnwidth]{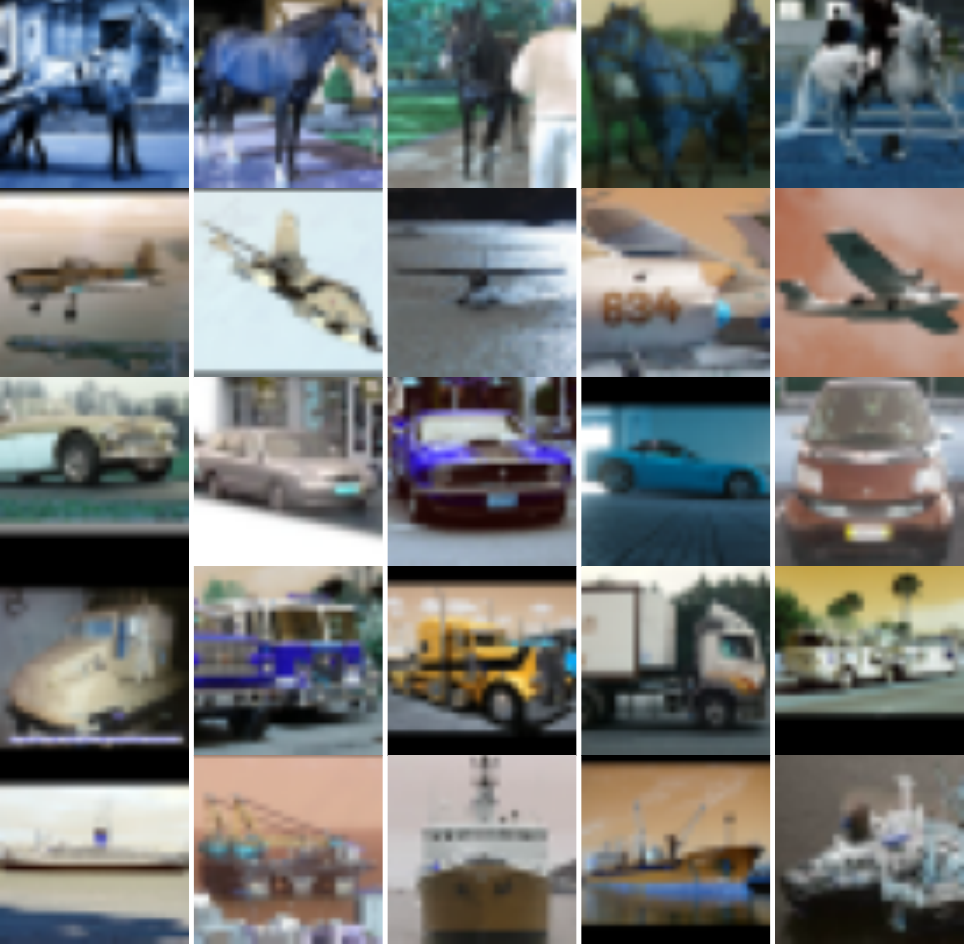} 
    \caption{Left: CIFAR-10 training examples. Right: Server's STL-10 examples.}
    \label{fig_synthetic_cifar}
\end{figure}

\begin{figure*}[!tb]
    \centering
    \includegraphics[width=\textwidth]{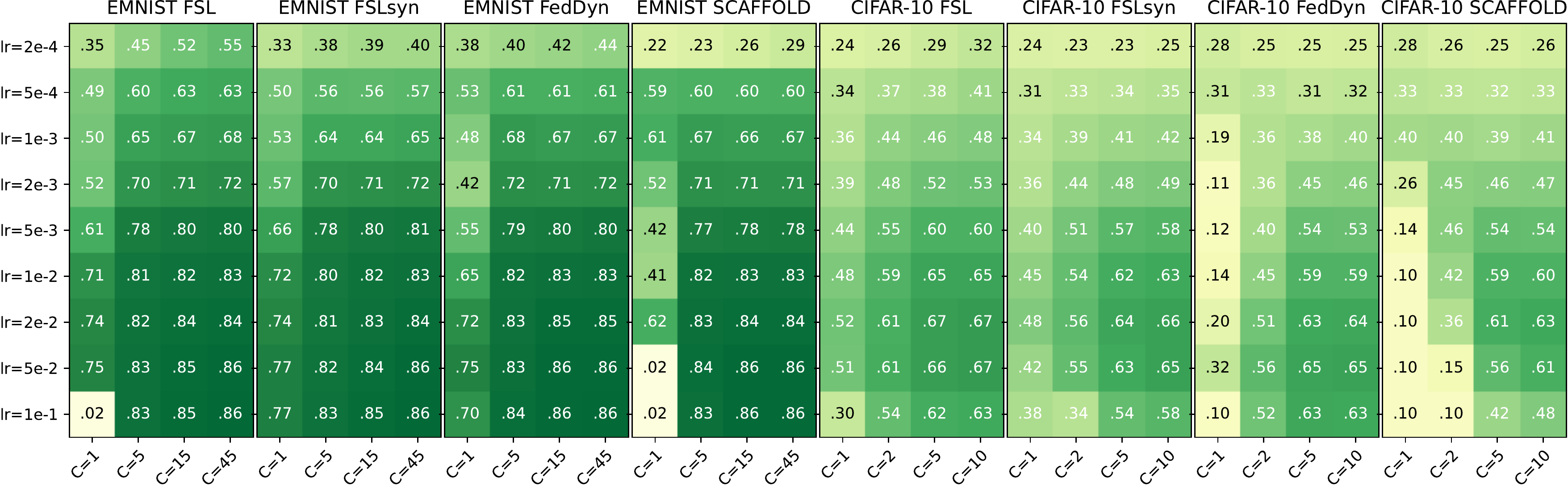}
    \caption{Heat maps of test accuracy when varying $\eta_l$ and $C$. Here, $B=50$, $S=5$ for EMNIST, and $B=10$, $S=10$ for CIFAR-10; $B$ is chosen following \cite{Karimireddy2020} so that 1 epoch of clients corresponds to 5 local steps.}
    \label{fig_acc_heatmap}
\end{figure*}

We compare \texttt{FSL} and \texttt{FSLsyn} (without using a pretrained model) against \texttt{SCAFFOLD} and \texttt{FedDyn} when $S=\lceil N/100\rceil$. We use $\eta_g = \sqrt{S}$ for \texttt{FSL}, \texttt{FSLsyn}, and \texttt{SCAFFOLD}.  
Fig.~\ref{fig_acc_heatmap} shows the test accuracy after $T=1,000$ rounds with varying learning rate $\eta_l$ and non-IIDness $C$. Here, we fix the weight $\gamma\!=\!1$ in \texttt{FSL} and \texttt{FSLsyn} and regularization parameter $\alpha=0.01$ in \texttt{FedDyn}; better performance can be obtained by tuning these parameters as we will show later. 
First, it shows that, compared to \texttt{SCAFFOLD} and \texttt{FedDyn}, our algorithms \texttt{FSL} and \texttt{FSLsyn}
have comparable overall accuracy for EMNIST and much better for CIFAR-10, especially in very non-IID cases,  even without tuning $\gamma$. The heatmap also suggests that it is fairly easy to select learning rates for \texttt{FSL} and \texttt{FSLsyn}. 
The results further indicate that using server learning with synthetic or other `good' sources of data can provide significant benefits. In fact, \texttt{FSLsyn} has comparable performance to \texttt{FSL} for EMNIST and slightly worse performance for \mbox{CIFAR-10} (but still better than \texttt{FedDyn} and \texttt{SCAFFOLD} in this case). 
Additional experimental results reported in 
Figures~\ref{fig_cifar_N100} and \ref{fig_emnist_N450} in Appendix~\ref{subsec_compare_FSL_SCAFFOLD} 
also show that \texttt{FSL} and \texttt{FSLsyn} have faster rise times in most cases. 
Note that our algorithm can be improved further by having more (and better) data for server learning and using a pretraining step.


\begin{table}[t]
\caption{Accuracy after 1k rounds and number of global rounds needed to reach 0.5 accuracy $T_{0.5}$ in CIFAR-10 with $C=2$.}
\label{sample-table}
\begin{center}
\begin{small}
\vskip -0.1in
\begin{tabular}{l|r|rr|rr}
\toprule
& \texttt{FedDyn} 
&\multicolumn{2}{c|}{\texttt{FSLsyn} ($\eta_l=0.02$)} 
&\multicolumn{2}{c}{\texttt{FSL} ($\eta_l=0.02$)} \\ 
& $\eta_l\!=\!0.05$ & $n_0\!=\!504$ & $n_0\!=\!720$ & $n_0\!=\!250$ & $n_0\!=\!500$ \\
\midrule
$T_{0.5}$ & 502  & 339 & 333 &  238 & 203 \\
Acc  & 0.5779    & 0.5763 & 0.5835  & 0.5845 & 0.6144 \\
\bottomrule
\end{tabular}
\end{small}
\end{center}
\vskip -0.1in
\end{table}

Finally, Table 1 shows that both the quantity and the quality of server's data $\mathcal{D}_0$ affect the performance of \texttt{FSL}. These results are obtained with CIFAR-10 when clients' data is highly non-IID with $C=2$, and we pick the learning rates according to the highest accuracy given in Fig.~\ref{fig_acc_heatmap}. We also fine-tune the regularization parameter of \texttt{FedDyn} with $\alpha\in \{0.01, 0.05, 0.1, 0.5\}$ following \cite{durmus2021federated} and the server weight $\gamma \in \{0.6, 0.8, 1.0, 1.2\}$ in \texttt{FSL} and \texttt{FSLsyn} -- we report the best numbers and skip \texttt{SCAFFOLD} as it underperforms \texttt{FedDyn}.   
Moreover, we vary the server data size $n_0\in \{250, 500\}$ for \texttt{FSL}, and $n_0\in \{504, 720\}$ for \texttt{FSLsyn}. Both the rise time and the accuracy improve as $n_0$ increases, with \texttt{FSL} featuring a more significant improvement since the server's data are more similar to the clients' data compared to synthetic data (see Fig.~\ref{fig_synthetic_cifar}). 
In addition, both of our algorithms require a significantly smaller number of global rounds to reach 0.5 accuracy, showcasing the benefit of server learning. It is also interesting to note that \texttt{FSL} with $n_0=250$ is still slightly better than \texttt{FSLsyn} with $n_0=720$, confirming that the synthetic data are likely taken from a different distribution.

\section{Conclusions}\label{sec_conclusion}
We considered a new approach to mitigate the performance degradation of FL on non-IID data. Our approach augments FL with server learning using a small dataset, and thus is complementary in that it can be utilized in conjunction with other existing approaches in the literature. 
Our analysis and experiments revealed that FSL can offer significant improvements in terms of accuracy and convergence time over conventional FL algorithms, even when the server dataset is relatively small. As expected, the improvements depend not only on server data size but also on the divergence between its distribution and that of the aggregate training data. The improvements are higher when the distributional divergence is smaller. We are currently exploring the issue of choosing a suitable dataset for the server learning and the relationship between the performance improvements and the server data size/the distributional divergence. 

\bibliography{ref}
\bibliographystyle{abbrv}


\begin{appendices}


\section{Proofs} \label{sec_Proofs}

Our proofs will use the following technical lemmas. 

\begin{lemma}\label{lem_SOS_RV}
If $\{z_1, z_2, \ldots, z_m\}$ are  independent random variables with 0-mean, then $\Ept \big[ \|\sum_{i=1}^m z_i\|^2 \big] = \Ept \big[ \sum_{i=1}^m \|z_i\|^2 \big]$
\end{lemma}

\begin{lemma}\emph{(CS inequality)}\label{CS-ineq}
The following hold for any $\{v_1,\ldots,v_m\} \subset \mathbb{R}^d$:
\begin{itemize}
    \item[1.] $\|v_i+v_j\|^2 \le (1+a)\|v_i\|^2 + (1+\frac{1}{a})\|v_j\|^2$ for any $a>0$, and
    \item[2.] $\|\sum_{i=1}^m v_i\|^2 \le m\sum_{i=1}^m\|v_i\|^2$.
\end{itemize}
\end{lemma}
We will refer to both inequalities above as the Cauchy-Schwarz (CS) inequality in the rest of this section.  

\subsection{Proof of Theorem~\ref{thm_descent_lem}} \label{proof_descent_FL}
Recall that our approximated global loss function is $\tilde{F} = \frac{1}{1+\gamma}(F+\gamma f_0)$. Our training algorithm is as follows. For any $t\ge 1$
\begin{align*}
    x^{(i)}_{t,k} &= x^{(i)}_{t,k-1} - \eta_l g^{(i)}_{t,k-1}, \quad \text{with } x^{(i)}_{t,0} = x_t, \quad \forall k\in [K], i\in \mathcal{S}_t\\
    \bar{x}_t &= x_t + \frac{\eta_g}{S}\sum_{i\in \mathcal{S}_t} \big( x^{(i)}_{t,K} - x_t\big) \nonumber\\
    w_{t,k} &= w_{t,k-1} -  \gamma\eta_0 g^{(0)}_{t,k-1}, \quad \text{with } w_{t,0} = \bar{x}_t, \quad \forall k\in [K_0]\\
    x_{t+1} &= w_{t,K_0} \nonumber
\end{align*}
where $\mathcal{S}_t$ is the random set of clients chosen to update the model at round $t$ with $S=|\mathcal{S}_t|$, $g^{(i)}_{t,k-1}$ is an unbiased estimate of $\nabla f_i(x^{(i)}_{t,k-1})$ for $i \in \mathcal{S}_t$, $g^{(0)}_{t,k-1}$ is an unbiased estimate of $\nabla f_0(w_{t,k-1})$, and the step sizes satisfy
\begin{align}
    K_0\eta_0 = K\eta_g\eta_l .
    \label{eqStepsizes}
\end{align}



Define
\begin{align*}
    E^{(c)}_t &= \Ept_t \left[ \frac{1}{KN}\sum_{i\in [N],k\in [K]} \big\| x_t - x^{(i)}_{t,k-1} \big\|^2 \right], \\ 
    E^{(0)}_t &= \Ept_t \left[ \frac{1}{K_0} \sum_{k\in [K_0]}\big\|x_t - w_{t,k-1} \big\|^2 \right] ,
\end{align*}
where $E^{(c)}_t$ is known as the drift caused by the clients' local updates, while $E^{(0)}_t$ is the drift due to server's updates in our algorithm. The following results are simply an application of the Lipschitz conditions of $\nabla f_i$ for $i=0,1,\ldots,N$. 
\begin{lemma}\label{lem_drifts}
We have the following relations:
\begin{align*}
    \Ept_t \left[ \sum_{k\in [K_0]}\big\|\nabla f_0(x_t) - \nabla f_0(w_{t,k-1}) \big\|^2 \right] \le K_0 L^2E^{(0)}_t  \\
    \frac{1}{N}\Ept_t \left[ \sum_{i\in [N], k\in [K]} \big\| \nabla f_i(x_t) - \nabla f_i(x^{(i)}_{t,k-1}) \big\|^2 \right] \le K L^2 E^{(c)}_t 
\end{align*}
\end{lemma}

From the $L$-smoothness of $\tilde{F}$, we have
\begin{align}
    &\Ept_{t} \left[ \tilde{F}({x}_{t+1}) \right] - \tilde{F}(x_t) 
    \le \underbrace{\langle \nabla \tilde{F}(x_t), \Ept_{t} \left[ x_{t+1} \right] - x_t \rangle}_{=:T_1} +  L \underbrace{\frac{1}{2}\Ept_{t} \left[ \|x_{t+1} - x_t\|^2 \right]}_{=:T_2} , 
    \label{descent_lem}
\end{align}
where the difference $x_{t+1} - x_t$ can be expressed as
\begin{align} 
&x_{t+1} - x_t 
= x_{t+1} - \bar{x}_t + \bar{x}_t- x_t 
=-K_0\eta_0 \Big( \frac{\sum_{k=1}^{K_0} \gamma g^{(0)}_{t,k-1}}{K_0} + \frac{\sum_{k=1}^{K}\sum_{i\in \mathcal{S}}  g^{(i)}_{t,k-1}}{KS} \Big). \label{eq_dx}
\end{align} 
Let us now bound the terms $T_1$ and $T_2$ on the right-hand siide of \eqref{descent_lem}. First, using \eqref{eq_dx},  \eqref{eqStepsizes}, and  the fact that $(1+\gamma)\tilde F = F + \gamma f_0$, 
we have
\begin{align}
    T_1 &= \inprod{ \nabla \tilde F(x_t)}{(1+\gamma)K_0\eta_0\nabla\tilde F(x_t) - \eta_0 \Ept_{t} \left[ \sum_{k=1}^{K_0} \gamma g^{(0)}_{t,k-1} \right] - \eta_g\frac{\eta_l}{S}\Ept_{t} \left[ \sum_{i\in \mathcal{S}} \sum_{k=1}^{K} g^{(i)}_{t,k-1} \right] - (1+\gamma)K_0\eta_0 \nabla \tilde F(x_t) } \nonumber\\
    &= \underbrace{\inprod{ \nabla \tilde F(x_t)}{\eta_0\gamma \Ept_{t} \left[ \sum_{k=1}^{K_0}\big( \nabla f_0(x_t) - g^{(0)}_{t,k-1}\big) \right]  + \eta_g\eta_l \Ept_{t} \left[ \sum_{k=1}^{K} \Big( \nabla F(x_t) - \frac{1}{S}\sum_{i\in \mathcal{S}} g^{(i)}_{t,k-1}  \Big) \right] } }_{=:T_3} \label{eq_T1_T3}\\
    &\quad - (1+\gamma)K_0\eta_0 \|\nabla \tilde F(x_t)\|^2 . \nonumber
\end{align}
Note that by taking expectation over $\mathcal{S}$ and using  Assumption~\ref{assm_localGrad}, we have
$\Ept_t \Big[ \frac{1}{S}\sum_{i\in \mathcal{S}} g^{(i)}_{t,k-1} \Big]$ $= \frac{1}{N}\Ept_t \Big[ \sum_{i \in [N]} \nabla f_i(x^{(i)}_{t,k-1}) \Big]$ and $\Ept_{t} \left[ \sum_{k}  g^{(0)}_{t,k-1} \right] = \Ept_{t} \big[ \sum_{k}\nabla f_0(w_{t,k-1}) \big].$ 
Using this, \eqref{eqStepsizes} and the fact that $F = \frac{1}{N}\sum_{i\in [N]}f_i$, we can bound $T_3$ as follows:
\begin{align*}
    \frac{2T_3}{K_0\eta_0} 
    &= \Ept_t \left[ \inprod{2\nabla \tilde F(x_t)}{   \frac{\sum_{i,k}\big( \nabla f_i(x_t) - \nabla f_i(x^{(i)}_{t,k-1}) \big)}{KN} + \frac{\sum_{k\in [K_0]}\gamma\big(\nabla f_0(x_t) - \nabla f_0(w_{t,k-1}) \big)}{K_0} } \right] \\
    &\le \|\nabla\tilde F(x_t)\|^2 
    + \Ept_t \left[ \Big\|  \frac{\sum_{i,k} \big( \nabla f_i(x_t) - \nabla f_i(x^{(i)}_{t,k-1}) \big)}{KN} + \frac{\sum_{k} \gamma\big( \nabla f_0(x_t) - \nabla f_0(w_{t,k-1}) \big)}{K_0} \Big\|^2 \right] \tag{CS ineq.}\\
    &\le \|\nabla\tilde F(x_t)\|^2 
    + 2\Ept_t \left[ \frac{ \sum_{i,k} \big\| \nabla f_i(x_t) - \nabla f_i(x^{(i)}_{t,k-1}) \big\|^2}{KN} + \frac{ \sum_{k}\gamma^2\big\|\nabla f_0(x_t) - \nabla f_0(w_{t,k-1}) \big\|^2}{K_0} \right] \tag{CS ineq.}\\
    &\le \|\nabla\tilde F(x_t)\|^2 
    + 2\big(L^2 E^{(c)}_t + L^2\gamma^2E^{(0)}_t\big) \tag{Lemma~\ref{lem_drifts}}
\end{align*}
Using this bound for \eqref{eq_T1_T3} yields 
\begin{align}
    T_1 \le K_0\eta_0\big( L^2 (E^{(c)}_t + \gamma^2E^{(0)}_t) -(0.5+\gamma)\|\nabla \tilde F(x_t)\|^2 \big). \label{eq_bound_T1}
\end{align}

Next, we bound the term $T_2$. Again, from \eqref{eq_dx} and the fact that $(1+\gamma)\tilde F = F + \gamma f_0$, we have 
\begin{align}
    \frac{2T_2}{3K_0^2\eta_0^{2}} &= \frac{1}{3}\Ept_{t} \left[ \bigg\| \frac{\sum_{k=1}^{K_0} \gamma g^{(0)}_{t,k-1}}{K_0} + \frac{\sum_{k=1}^{K}\sum_{i\in \mathcal{S}}  g^{(i)}_{t,k-1}}{KS} \bigg\|^2 \right]  \nonumber\\
    &=\frac{1}{3} \Ept_{t} \left[ \Big\|\frac{\sum_{k=1}^{K_0} \gamma g^{(0)}_{t,k-1}}{K_0} - \gamma\nabla f_0(x_{t}) + \frac{\sum_{k=1}^{K}\sum_{i\in \mathcal{S}}  g^{(i)}_{t,k-1}}{KS} -\nabla F(x_{t}) + (1+\gamma)\nabla\tilde{F}(x_t)\Big\|^2 \right] \nonumber\\
    &\le \gamma^2\underbrace{\Ept_{t} \Big\| \Big(\frac{1}{K_0}\sum_{k\in [K_0]} g^{(0)}_{t,k-1}\Big) - \nabla f_0(x_{t}) \Big\|^2 }_{=:T_4} 
        + \underbrace{\Ept_{t}  \Big\|\Big(\frac{1}{KS}\sum_{i\in \mathcal{S}}\sum_{k} g^{(i)}_{t,k-1} \Big) -\nabla F(x_{t})\Big\|^2 }_{=:T_5} 
        + (1+\gamma)^2\big\|\nabla\tilde{F}(x_t)\big\|^2 , \label{eq_T2_T4T5}
\end{align}
where the last inequality follows from the Cauchy-Schwarz inequality. We first consider $T_4$.
\begin{align*}
    \frac{T_4}{2} 
    &= \frac{1}{2}\Ept_{t} \left[ \Big\|\Big(\frac{1}{K_0}\sum_{k}  g^{(0)}_{t,k-1} \Big) - \nabla f_0(x_{t})\Big\|^2 \right] \\
    &= \frac{1}{2K_0^2}\Ept_{t} \left[ \Big\|\sum_{k} \big(g^{(0)}_{t,k-1} - f_0(w_{t,k-1})\big) + \sum_{k} \big(f_0(w_{t,k-1}) - \nabla f_0(x_{t})\big) \Big\|^2 \right] \\
    &\le \underbrace{K_0^{-2}\Ept_{t} \left[ \Big\|\sum_{k} \big(g^{(0)}_{t,k-1} - \nabla f_0(w_{t,k-1})\big) \Big\|^2 \right] }_{=:T_{4a}}
        + \underbrace{K_0^{-2}\Ept_{t} \left[ \Big\|\sum_{k} \big(f_0(w_{t,k-1}) - \nabla f_0(x_{t})\big) \Big\|^2 \right] }_{=:T_{4b}}
\end{align*}
Here, by Lemma~\ref{lem_SOS_RV} and Assumption~\ref{assm_localGrad}, we have 
\begin{align}
    T_{4a} = K_0^{-2}\Ept_{t} \left[ \sum_{k\in [K_0]} \big\| g^{(0)}_{t,k-1} - \nabla f_0(w_{t,k-1}) \big\|^2 \right]
            \le \frac{\sigma_0^2}{K_0} . \nonumber
\end{align}
Applying the Cauchy-Schwarz inequality to $T_{4b}$ yields
\begin{align}
    T_{4b} \le \Ept_{t} \left[ \frac{\sum_{k\in [K_0]} \big\| \nabla f_0(w_{t,k-1}) - \nabla f_0(x_{t}) \big\|^2}{K_0} \right] \le L^2E^{(0)}_t \tag{Lemma~\ref{lem_drifts}} . 
\end{align}
Thus, from the above bounds, 
\begin{align}
    T_4 \le 2T_{4a} + 2T_{4b} = \frac{2\sigma_0^2}{K_0} + 2L^2E^{(0)}_t.
\end{align}

Similarly, we can bound $T_5$ as follows:
\begin{align*}
    K^2T_{5} 
    &= \Ept_{t} \left[ \Big\| \Big(\sum_{i\in \mathcal{S},k\in [K]} \frac{1}{S} g^{(i)}_{t,k-1} \Big) - K \nabla F(x_t) \Big\|^2 \right] 
    = \Ept_{t} \left[ \Big\| \frac{1}{S}\sum_{i\in \mathcal{S},k\in [K]} \big( g^{(i)}_{t,k-1} - \nabla F(x_t)\big) \Big\|^2 \right] \\
    &=  \Ept_{t} \left[ \Big\| \frac{1}{S}\sum_{i\in \mathcal{S},k\in [K]} \big( g^{(i)}_{t,k-1} -\nabla f_i(x^{(i)}_{t,k-1}) + \nabla f_i(x^{(i)}_{t,k-1}) - \nabla f_i(x_{t}) + \nabla f_i(x_{t})-  \nabla F(x_t)\big) \Big\|^2 \right]
\end{align*}
Rearranging terms and applying the Cauchy-Schwarz inequality yields
\begin{align*}
    \frac{K^2T_5}{3} 
&\le \underbrace{\Ept_{t} \left[ \Big\| \frac{1}{S}\sum_{i\in \mathcal{S},k\in [K]}  \big( g^{(i)}_{t,k-1} -\nabla f_i(x^{(i)}_{t,k-1}) \big) \Big\|^2 \right] }_{=: T_{5a}}
+ \underbrace{ \Ept_{t} \left[ \Big\| \frac{1}{S}\sum_{i\in \mathcal{S},k\in [K]} \big( \nabla f_i(x^{(i)}_{t,k-1}) - \nabla f_i(x_{t}) \big) \Big\|^2 \right] }_{=: T_{5b}} \\
&\quad + \underbrace{ \Ept_{t} \left[ \Big\| \frac{1}{S}\sum_{i\in \mathcal{S},k\in [K]} \big(\nabla f_i(x_{t})- \nabla F(x_t) \big) \Big\|^2 \right]  }_{=: T_{5c}} .  
\end{align*}
Each term on the RHS can be bounded as follows. First, by using  Lemma~\ref{lem_SOS_RV} and Assumption~\ref{assm_localGrad}, we have 
\begin{align*}
    T_{5a} \stackrel{\text{(Lem.~\ref{lem_SOS_RV})}}{=} \frac{1}{S^2} \Ept_t \left[ \sum_{i\in \mathcal{S},k\in [K]} \| g^{(i)}_{t,k-1} -\nabla f_i(x^{(i)}_{t,k-1}) \|^2 \right]  \stackrel{\text{(Assump. \ref{assm_localGrad})}}{\le} \frac{K}{S}\sigma^2, 
\end{align*}
\begin{align*}
    T_{5b} 
    &\le  \frac{K}{S} \Ept_t \left[ \sum_{i\in \mathcal{S},k\in [K]} \| \nabla f_i(x^{(i)}_{t,k-1}) - \nabla f_i(x_{t}) \|^2 \right] \tag{CS ineq.}\\
    &\le \frac{KL^2}{S} \Ept_t \left[ \sum_{i\in \mathcal{S},k\in [K]} \| x^{(i)}_{t,k-1} - x_{t} \|^2 \right] \tag{$L$-smooth.}\\
    &= K^2L^2\underbrace{\frac{1}{KN} \Ept_t \left[ \sum_{i,k} \| x^{(i)}_{t,k-1} - x_{t} \|^2 \right] }_{E^{(c)}_t} = K^2L^2E^{(c)}_t , \tag{Exp. on $\mathcal{S}$}
\end{align*}
and due to sampling without replacement, 
\begin{align*}
    T_{5c} &= K^2 \Ept_t \left[ \Big\|\frac{\sum_{i\in \mathcal{S}} \nabla f_i(x_{t})}{S}-  \nabla F(x_t) \Big\|^2 \right] 
    = \frac{K^2}{S} \left( 1-\frac{S}{N} \right) \frac{\sum_{i\in [N]} \|\nabla f_i(x_{t})-  \nabla F(x_t)\|^2}{N-1}\\
    &\le \frac{K^2}{S} \left( 1-\frac{S}{N} \right) \frac{N G^2}{N-1} 
    \le K^2\tau_s G^2 , \tag{Assump.~\ref{assm_globalGrad}} 
\end{align*}
where $\tau_s = \frac{(N-S)}{S(N-1)}$. Thus, 
\begin{align}
    T_5\le 3K^{-2}(T_{5a}+T_{5b}+T_{5c})
    &\le 
    3\left( \frac{\sigma^2}{KS} +\tau_s G^2 +  L^2E^{(c)}_t \right) . \label{eq_bound_T5}
\end{align}

Combining the bounds above for $T_4$ and $T_5$, we have
\begin{align}
    T_2 \le 1.5K_0^2\eta_0^2 \left( \frac{2\gamma^2\sigma_0^2}{K_0} + 2L^2\gamma^2E^{(0)}_t 
                        + (1+\gamma)^2\big\|\nabla\tilde{F}(x_t)\big\|^2 
                        + 3\Big( \frac{\sigma^2}{KS} +  L^2E^{(c)}_t + \tau_s G^2 \Big)\right) . \label{eq_bound_T2}
\end{align}
Using this bound and \eqref{eq_bound_T1} for \eqref{descent_lem}, we obtain
\begin{align}
    \Ept_{t} \left[ \tilde{F}({x}_{t+1}) \right] 
    &\le \tilde{F}(x_t) - K_0\eta_0\big(0.5+\gamma - 1.5(1+\gamma)^2K_0\eta_0 L\big) \|\nabla \tilde F(x_t)\|^2 \nonumber  \\
    &\quad      + K_0\eta_0 L^2\gamma^2 (1+3K_0\eta_0L)E^{(0)}_t   + K_0 \eta_0 L^2 (1+4.5K_0\eta_0 L)E^{(c)}_t \nonumber\\
    &\quad      + K_0^2\eta_0^2L\underbrace{\Big( \frac{3\gamma^2\sigma_0^2}{K_0} + \frac{9\sigma^2}{2KS} + \frac{9\rho_s G^2}{2S} \Big)}_{=:G_2} . \label{eq_descent_G2}
\end{align}
Note that the step size condition in \eqref{eq_stepsize_mainCond} implies that $1+3K_0\eta_0L < 1+4.5K_0\eta_0L \le 2$. We then have
\begin{align}
    \Ept_{t} \left[ \tilde{F}({x}_{t+1}) \right] 
    &\le \tilde{F}(x_t) - K_0\eta_0\big( 0.5+\gamma - 1.5(1+\gamma)^2K_0\eta_0 L \big)\|\nabla \tilde F(x_t)\|^2   + K_0^2\eta_0^2 L G_2 \nonumber\\
    &\quad      + 2K_0\eta_0 L^2\underbrace{\Big( \gamma^2 E^{(0)}_t   +  E^{(c)}_t \Big)}_{=:T_6} . \label{eq_descent_T6}
\end{align}
To bound $T_6$, let us use 
the following results for bounding the drift terms above; 
the proofs of which are give in the next section below.  
\begin{lemma}\label{lem_client_drift}
If $4K\eta_l L \le 1$, then
\begin{align}
    E^{(c)}_t \le 4K^2\eta_l^2 \Big(\|\nabla F(x_t)\|^2 + \frac{\sigma^2}{2K} + G^2 \Big) \label{eq_client_drift}
\end{align}
\end{lemma}

\begin{lemma}\label{lem_server_drift}
If $4K_0\eta_0\gamma L \le 1$, then
\begin{align}
E^{(0)}_t 
&\le 12K_0^2\eta_0^2 \Big( L^2 E^{(c)}_t  + \|\nabla F(x_t)\|^2 + \frac{4}{3}\gamma^2\|\nabla f_0(x_t)\|^2  + G_3 \Big),\label{eq_server_drift}
\end{align}
where $G_3:=\frac{\sigma^2}{KS} + \frac{\rho_s G^2}{S} + \frac{\gamma^2\sigma_0^2}{3K_0}$.
\end{lemma}

Using the results above, we can continue to bound $T_6$ in \eqref{eq_descent_T6} as follows. 
\begin{align*}
    T_6 = \gamma^2 E^{(0)}_t   + E^{(c)}_t 
    &\le \underbrace{(1 + 12K_0^2\eta_0^2L_0^2) E^{(c)}_t}_{=:T_7}   + 12K_0^2\eta_0^2\gamma^2  \Big( \|\nabla F(x_t)\|^2 + \frac{4}{3}\gamma^2\|\nabla f_0(x_t)\|^2  + G_3 \Big) , 
\end{align*}
where $L_0=\gamma L$. Under condition \eqref{eq_stepsize_mainCond}, we have $12K_0^2\eta_0^2L_0^2 \le 1$. Then, 
\begin{align}
    T_7 &\le 2 E^{(c)}_t \le 8K^2\eta_l^2 \Big( \|\nabla F(x_t)\|^2 + \frac{\sigma^2}{2K} + G^2 \Big) . \tag{cf. \eqref{eq_client_drift} }
\end{align}
As a result, 
\begin{align*}
    T_6 &= \gamma^2 E^{(0)}_t   + E^{(c)}_t \\
    &\le 8K_0^2\eta_0^2\eta_g^{-2}\Big( \|\nabla F(x_t)\|^2 + \frac{\sigma^2}{2K} + G^2 \Big) \tag{using $K_0\eta_0 = K\eta_l \eta_g$ }\\
    &\quad + 12K_0^2\eta_0^2\gamma^2  \Big( \|\nabla F(x_t)\|^2 + \frac{4}{3}\gamma^2\|\nabla f_0(x_t)\|^2  + G_3 \Big)\\
    &\le K_0^2\eta_0^2 \|\nabla F(x_t)\|^2 \big( 8\eta_g^{-2} +  12\gamma^2\big)  + K_0^2\eta_0^2 \gamma\|\nabla f_0(x_t)\|^2 \big(16 \gamma^3\big)\\
    &\quad + 4K_0^2\eta_0^2 \underbrace{\big( 3\gamma^2G_3 + \eta_g^{-2}(2G^2+\sigma^2K^{-1}) \big)}_{=:G_4} \\
    &\le 4\kappa K_0^2\eta_0^2\big(\|\nabla F(x_t)\|^2 +\gamma\|\nabla f_0(x_t)\|^2\big) + 4K_0^2\eta_0^2 G_4 \tag{$\kappa = \max\{4 \gamma^3,2\eta_g^{-2} +  3\gamma^2\}$}
    %
\end{align*}
Again, since $\tilde F = \frac{1}{(1+\gamma)}(F+\gamma f_0)$, we have
$$\|\nabla F(x_t)\|^2 + \gamma \|\nabla f_0(x_t)\|^2 = (1+\gamma)\|\nabla \tilde F(x_t)\|^2 +\frac{\gamma }{1+\gamma }\underbrace{\|\nabla F(x_t) - \nabla f_0(x_t)\|^2}_{\xi^2(x_t)}.$$ 
Therefore, 
\begin{align*}
    T_6 \le 4\kappa K_0^2\eta_0^2 \Big( (1+\gamma )\|\nabla \tilde F(x_t)\|^2 +\frac{\gamma}{1+\gamma} \xi^2(x_t)\Big) + 4K_0^2\eta_0^2G_4, 
\end{align*}
which, in light of \eqref{eq_descent_T6}, implies
\begin{align}
    \Ept_{t} \left[ \tilde{F}({x}_{t+1}) \right] 
    &\le \tilde{F}(x_t) - K_0\eta_0(0.5+\gamma - 1.5(1+\gamma)^2K_0\eta_0 L)\|\nabla \tilde F(x_t)\|^2   + K_0^2\eta_0^2 L G_2 + 2K_0\eta_0L^2T_6 \nonumber\\
    &\le \tilde{F}(x_t) - \frac{K_0\eta_0}{2} \Big(2\gamma +1 - K_0\eta_0 L(1+\gamma)\big( 3(\gamma+1) + 16\kappa K_0\eta_0L \big) \Big)\|\nabla \tilde F(x_t)\|^2   \nonumber\\
    &\quad + K_0^2\eta_0^2 LG_2 + 8K_0^3\eta_0^3L^2\big( \textstyle\frac{\gamma\kappa}{1+\gamma} \xi^2(x_t) + G_4 \big) .
\end{align}
The proof is then completed by noting that $G_3 \le \frac{2}{9}G_2 \le \Psi$ and $G_4 \le \Phi$. 

\subsection{Proof of Lemma~\ref{lem_client_drift}}

The proof follows the same line of arguments as in the proof of Lemma~8 in \cite{Karimireddy2020}; we provide it here 
for completeness and for later reference in the proof of Lemma~\ref{lem_server_drift}. 

For simplicity, we drop the index $t$ in this proof, including conditional expectation $\Ept_t$. Clearly, the result holds for $K=1$ and thus we consider only $K\ge 2$ below. 
\begin{align*}
    \Ept \left[ \big\| x^{(i)}_{k} - x \big\|^2 \right] 
    &= \Ept \big\| x^{(i)}_{k-1} - x - \eta_l g^{(i)}_{k-1}\big\|^2 
    \le \Ept \big\| x^{(i)}_{k-1} - x - \eta_l \nabla f_i(x^{(i)}_{k-1})\big\|^2 + \eta_l^2 \sigma^2 \\
    &\le \Big( 1+ \frac{1}{K-1}\Big) \Ept \left[ \big\| x^{(i)}_{k-1} - x \big\|^2 \right] + K\eta_l ^2\big\|\nabla f_i(x^{(i)}_{k-1})\big\|^2 + \eta_l^2 \sigma^2 \tag{CS ineq.} \\
    &\le \frac{K}{K-1}\Ept \big\| x^{(i)}_{k-1} - x \big\|^2 + 2K\eta_l ^2\big\|\nabla f_i(x^{(i)}_{k-1}) - \nabla f_i(x)\big\|^2 + 2K\eta_l ^2\big\|\nabla f_i(x)\big\|^2 + \eta_l^2 \sigma^2 \tag{CS ineq.} \\
    &\le \Big(\frac{K}{K-1} + 2K\eta_l^2L^2\Big) \Ept \left[ \big\| x^{(i)}_{k-1} - x \big\|^2 \right] + 2K\eta_l^2\big\|\nabla f_i(x)\big\|^2 + \eta_l^2 \sigma^2 \tag{$L$ smooth.} \\
    &\le (1+a )\Ept \left[ \big\| x^{(i)}_{k-1} - x \big\|^2 \right] + \eta_l^2\big( 2K\big\|\nabla f_i(x)\big\|^2 + \sigma^2\big) ,\tag{$a:=\frac{1.125}{K-1}$}
\end{align*}
where the last inequality holds because $4K\eta_lL\le 1$  and $2K\eta_l^2L^2 = \frac{(4K\eta_lL)^2}{8K}< \frac{1}{8(K-1)}$ for any $K$. Unrolling the relation above
\begin{align}
\Ept \left[ \big\| x^{(i)}_{k} - x \big\|^2 \right]
    &\le \eta_l^2\big( 2K\big\|\nabla f_i(x)\big\|^2 + \sigma^2\big) \sum_{k=0}^{K-1} (1+a)^k \le 1.85K \eta_l^2\big( 2K\big\|\nabla f_i(x)\big\|^2 + \sigma^2\big) , \label{eq_client_i_drift}
\end{align}
where the last inequality holds since $a=\frac{1.125}{K-1}$ and $\sum_{k=0}^{K-1}\frac{ (1+a)^k}{K} = \frac{(1+a)^K-1}{aK} < \frac{e^{1.125}-1}{1.125} < 1.85$ for any $K\ge 2$.  
Thus, averaging the above relation over $k$ and $i$ yields
\begin{align*}
\Ept \left[ \frac{1}{KN}\sum_{i,k}\big\| x^{(i)}_{k} - x \big\|^2 \right]
    \le 3.7K^2\eta_l^2\Big(\frac{\sum_{i}\|\nabla f_i(x)\|^2}{N} + \frac{\sigma^2}{2K}\Big) 
    \le 3.7K^2\eta_l^2\Big( \|\nabla F (x)\|^2 + G^2+ \frac{\sigma^2}{2K}\Big). 
\end{align*}


\subsection{Proof of Lemma~\ref{lem_server_drift}}
Note that $w_{t,0} = \bar{x}_t$ and 
\begin{align} 
E^{(0)}_t &= \Ept_t \left[ \frac{1}{K_0} \sum_{k\in [K_0]}\big\|x_t - w_{t,k-1} \big\|^2 \right]
\le \Ept_t \left[ \frac{2}{K_0} \sum_{k\in [K_0]} \Big( \big\|x_t - \bar{x}_t\big\|^2 +\big\| \bar{x}_t - w_{t,k-1} \big\|^2 \Big) \right] \tag{CS ineq.} \\
&= 2\Ept_t \left[ \big\|x_t - \bar{x}_t\big\|^2 \right] + \frac{2}{K_0}\Ept_t \left[ \sum_{k\in [K_0]}\big\| w_{t,k-1} - w_{t,0} \big\|^2 \right] . 
\end{align}
Following the same line of arguments to obtain \eqref{eq_client_i_drift} as in the proof of Lemma~\ref{lem_client_drift}, we have
\begin{align}
    \Ept_t \left[ \frac{1}{K_0}\sum_{k\in [K_0]}\big\| w_{t,k-1} - w_{t,0} \big\|^2 \right]
    \le 4K_0^2\eta_0^2\gamma^2 \big( \|\nabla f_0(\bar{x}_t)\|^2 + \frac{\sigma_0^2}{2K_0} \big) .
\end{align}
Note that
\begin{align}
    \|\nabla f_0(\bar{x}_t)\|^2 
    &= \|\nabla f_0(x_t) + \nabla f_0(\bar{x}_t) - \nabla f_0(x_t) \|^2 
    \le 2\|\nabla f_0(x_t)\|^2 + 2\|\nabla f_0(\bar{x}_t) - \nabla f_0(x_t) \|^2 \tag{CS ineq.} \\
    &\le  2\|\nabla f_0(x_t)\|^2 + 2L^2 \|\bar{x}_t - x_t \|^2 \tag{$L$-smooth.}
\end{align}
Therefore, 
\begin{align*} 
E^{(0)}_t 
&\le (2+4^2K_0^2\eta_0^2L^2\gamma^2)\Ept_t\big\|x_t - \bar{x}_t\big\|^2 + 4^2K_0^2\eta_0^2\gamma^2\|\nabla f_0(x_t)\|^2 + 4K_0\eta_0^2 \gamma^2\sigma_0^2 \\
&\le 3\Ept_t \left[ \big\|x_t - \bar{x}_t\big\|^2 \right] + 4^2K_0^2\eta_0^2\gamma^2\|\nabla f_0(x_t)\|^2 + 4K_0\eta_0^2 \gamma^2\sigma_0^2 . \tag{cf. $4K_0\eta_0\gamma L \le 1$}
\end{align*}

Next, let us consider the term $\Ept_t\big\|\bar{x}_t - x_t\big\|^2$. Note that 
\begin{align*}
    \frac{1}{K^2\eta_g^2\eta_l^2} \Ept_t \left[ \big\|\bar{x}_t - x_t\big\|^2 \right] 
    &= \Ept_{t} \left[ \Big\| \Big( \frac{1}{KS}\sum_{i\in \mathcal{S}, k\in [K]} g^{(i)}_{t,k-1}\Big) - \nabla F(x_t) + \nabla F(x_t) \Big\|^2 \right] \\
    &\le 4\|\nabla F(x_t)\|^2 + \frac{4}{3}\underbrace{\Ept_{t} \left[ \Big\| \Big( \frac{1}{KS}\sum_{i\in \mathcal{S}, k\in [K]} g^{(i)}_{t,k-1}\Big) - \nabla F(x_t)\Big\|^2 \right] }_{=T_5 \text{ in \eqref{eq_T2_T4T5}} } \tag{CS ineq.} \\
    &\le 4\Big( (KS)^{-1}\sigma^2 + \tau_sG^2+ L^2E^{(c)}_t + \|\nabla F(x_t)\|^2 \Big) .  \tag{cf.~\eqref{eq_bound_T5}}
\end{align*}
Therefore, 
$\Ept_t \left[ \big\|\bar{x}_t - x_t\big\|^2 \right] \le 4K_0^2\eta_0^2 \Big( \frac{\sigma^2}{KS} + \tau_sG^2+ L^2E^{(c)}_t + \|\nabla F(x_t)\|^2 \Big)$. Thus, 
\begin{align*} 
E^{(0)}_t 
&\le 12K_0^2\eta_0^2 \Big( L^2 E^{(c)}_t  + \|\nabla F(x_t)\|^2 + \frac{4}{3}\gamma^2\|\nabla f_0(x_t)\|^2  + G_3 \Big), \quad \text{with } G_3=\frac{\sigma^2}{KS} + \tau_sG^2 + \frac{\gamma^2\sigma_0^2}{3K_0}. 
\end{align*}



\subsection{Proof of Theorem \ref{thm_progress_overall}}
From \eqref{eq_descent_lemma} we obtain 
\begin{align}
    h\Ept \|\nabla \tilde F(x_t)\|^2 
    &\le \frac{\Ept \tilde{F}(x_t) - \Ept \tilde{F}({x}_{t+1})}{K_0\eta_0}   + 5K_0\eta_0 L\Psi  + 8K_0^2\eta_0^2L^2\big( \textstyle\frac{\gamma\kappa}{1+\gamma} \bar{\xi}^2 + \Phi \big) \nonumber\\
    &= \frac{\Ept \big[ \tilde{F}(x_t) - \tilde{F}^*\big] - \Ept \big[ \tilde{F}({x}_{t+1}) - \tilde{F}^*\big] }{K_0\eta_0}   + 5K_0\eta_0 L\Psi  + 8K_0^2\eta_0^2L^2\big( \textstyle\frac{\gamma\kappa}{1+\gamma} \bar{\xi}^2 + \Phi \big) . \nonumber 
\end{align}
Summing this relation over $t = 0,\ldots,T-1$ and then simplifying terms yields the desired result.

\subsection{Corollary~\ref{coro:1} and its proof} \label{subsec_coro_proof}
The following results show the convergence error for different step size conditions, which subsume Corollary~\ref{coro:1}. 
\begin{corollary}
Assume that condition \eqref{eq_stepsize_mainCond2} is satisfied. 
\begin{itemize}
    \item[(a)] If $K_0\eta_0 = \Theta(1/(\gamma+1)\sqrt{T})$, then
    \begin{align}
        \min_{0\le t \le T-1}\Ept \|\nabla \tilde F(x_t)\|^2 = \mathcal{O} \left( \frac{\tilde{D}_0}{\sqrt{T}}   + \frac{L\Psi}{\sqrt{T}(1+\gamma)^2} + \frac{L^2\gamma\kappa\bar{\xi}^2 }{T(1+\gamma)^4} + \frac{L^2\Phi}{T(1+\gamma)^3} \right), \nonumber
    \end{align}
    where $\tilde{D}_0 =\tilde{F}(x_0) - \tilde{F}^*$.
    
    \item[(b)]  If $\eta_g = \Theta(\sqrt{S})$ and $K_0\eta_0 = \Theta \Big( \frac{\sqrt{S}}{\sqrt{LT}(\gamma+1)} \Big)$, then 
    \begin{align}
        \min_{0\le t \le T-1} \Ept \|\nabla \tilde F(x_t)\|^2 
        &= \mathcal{O} \left( \frac{\sqrt{L}}{\sqrt{ST}}   \big( \tilde{D}_0+\frac{\tilde{G}^2}{(\gamma+1)^2}\big)  + \frac{L}{T} \big( \frac{\tilde{G}^2}{1+\gamma} + \frac{\gamma\kappa S\bar{\xi}^2}{(1+\gamma)^4} \big) \right), \nonumber
    \end{align}
    where  $\tilde{G}^2 = \frac{\gamma^2\sigma_0^2S}{K_0} + \frac{\sigma^2}{K}+ \rho_s G^2$.

    \item[(c)] If $\eta_g = \Theta(\sqrt{S})$, $K_0 = \Theta(K)$ and $K_0\eta_0 = \Theta \Big( \frac{\sqrt{KS}}{\sqrt{LT}(\gamma+1)} \Big)$, then 
    \begin{align}
        \min_{0\le t \le T-1} \Ept \|\nabla \tilde F(x_t)\|^2 
        &= \mathcal{O} \left( \frac{\sqrt{L}}{\sqrt{KST}}   \Big( \tilde{D}_0+\frac{M^2}{(\gamma+1)^2} \Big)  + \frac{L}{T} \Big( \frac{M^2}{1+\gamma} + \frac{\gamma\kappa KS\bar{\xi}^2}{(1+\gamma)^4} \Big) \right) \nonumber
    \end{align}
    with $M^2 = \gamma^2\sigma_0^2S +\sigma^2 + \rho_s KG^2$.
\end{itemize}
\end{corollary}

\begin{proof}
Note that under condition \eqref{eq_stepsize_mainCond2}, we have $h = \Omega(\gamma+1)$. Thus, the proof follows immediately from dividing both sides of the inequality in Theorem~\ref{thm_progress_overall} by $h$ and then using the step size conditions in each statement. 
\end{proof}

\subsection{Special Cases} \label{sec_FSL_special_cases}
Let us show here that our FSL approach includes FedAvg and centralized SGD as special cases.

\paragraph{FedAvg as a Special Case} 
Clearly, FedAvg is the special case of our formulation with $\gamma \to 0$, i.e., there is no server learning. Take Corollary 3.7, for example, although the bound on the RHS is by no means tight, it’s clear that using a positive $\gamma$ will lower RHS (provided $\bar{\xi}$ small compared to $G$), especially the coefficient of the dominant error term $\frac{1}{\sqrt{T}}$ in the bound, which is $\tilde{D}_0 + \frac{G^2}{(1+\gamma)^2}$. Note also that $\tilde{D}_0 = \mathcal{O}(\frac{D_0}{1+\gamma})$ for small $\gamma$. We briefly discussed this after Corollary 3.7. 

\paragraph{Centralized SGD as a Special Case} Consider the case where Server has access to all training data of clients and thus can just perform local/centralized learning. This extreme case can be approximated by FSL using $\bar{\xi}=0$ and a sufficiently large value of $\gamma$. Since Corollary 3.7 is for $\gamma = \mathcal{O}(1)$, we can use Theorem 3.6 instead. 

For large $\gamma$, we have $\kappa = \Theta(\gamma^3)$. From (9), let us choose $K_0\eta_0L =\Theta(1/\gamma^2)$, which implies $h =\Theta(\gamma)$, $\Phi =\Theta(\gamma^2 \Psi)$, and $\Psi =\Theta( \frac{\gamma^2 \sigma_0^2}{K_0} + \tilde{G})$ with $\tilde{G} = \frac{\sigma^2}{KS} + \frac{\rho_s G^2}{S}$. 
Let $\alpha = \eta_0 \gamma$, which is the stepsize of server in Algorithm 1. 
As a result, Theorem 3.6 implies
$\mathcal{E}_T = \mathcal{O}(\frac{\tilde{D}}{TK_0 \alpha} + L\sigma_0^2\alpha + \frac{\tilde{G}}{\gamma^3})$. The first two terms are exactly the error bound of SGD with a fixed stepsize, and the last term is the contribution of clients, which is negligible for large $\gamma$.

\section{Further Numerical Results} \label{sec_further_result}
In this section, we present detailed models and further simulation results. 

\subsection{Neural Network Models Used in Our Experiments}\label{subsec_NN_models}
In our experiments, we use networks with 2 convolutional layers followed by 2 dense layers as shown in Tables \ref{tab_EMNIST_model} and  \ref{tab_CIFAR10_model}. Note that these models are enough for our purpose of illustrations and comparing different algorithms; they are by no means designed to achieve the state-of-the-art accuracy. 
\begin{table}[tbh]
    \caption{Model used in EMNIST experiments}
    \label{tab_EMNIST_model}
    \vskip 0.15in
    \centering
    \begin{tabular}{ccccc}
    \toprule
        Layer & Output Shape & Param. \#  & Activation & Hyper-param. \\
    \midrule
        Input & $(28, 28, 1)$ &      &  &  \\
        Conv2D & $(28, 28, 32)$ &     $320$  & relu & kernel size = 3, stride = $(1,1)$ \\
        Conv2D & $(26, 26, 64)$ &     $18,496$ & relu & kernel size = 3, stride = $(1,1)$ \\ 
        MaxPooling2D & $(13, 13, 64)$ & & & pool size = $(2, 2)$\\
        Dropout & $(13, 13, 64)$ & & & $p=0.25$ \\
        Flatten & $10,816$ & & & \\
        Dense   &  $128$  & $1,384,576$ & relu & \\
        Dropout & $128$  & & & $p=0.5$ \\
        Dense   & $45$  & $5,805$ & softmax & \\
    \bottomrule
    \end{tabular}
    \vskip -0.1in
\end{table}

\vskip -0.1in
\begin{table}[tbh]
\caption{Model used in CIFAR-10 experiments}
    \label{tab_CIFAR10_model}
    \centering
    \begin{tabular}{ccccc}
    \toprule
        Layer & Output Shape & Param. \#  & Activation & Hyper-param. \\
    \midrule 
        Input & $(32, 32, 3)$ &      &  &  \\
        Conv2D & $(32, 32, 32)$ &     $896$  & relu & kernel size = 3, stride = $(1,1)$ \\
        MaxPooling2D & $(13, 13, 64)$ & & & pool size = $(2, 2)$\\
        Conv2D & $(16, 16, 64)$ &     $18,496$ & relu & kernel size = 3, stride = $(1,1)$ \\ 
        MaxPooling2D & $(8, 8, 64)$ & & & pool size = $(2, 2)$\\
        Dropout & $(8, 8, 64)$ & & & $p=0.25$ \\
        Flatten & $4,096$ & & & \\
        Dense   &  $128$  & $524,416 $ & relu & \\
        Dropout & $128$  & & & $p=0.5$ \\
        Dense   & $10$  & $1,290$ & softmax & \\
    \bottomrule
    \vspace{1ex}
    \end{tabular}
    \vskip -0.1in
\end{table}

\subsection{Futher Results}
\subsubsection{Test Accuracy Without Using Server Pretrained Model}\label{subsec_wo_Pretrain}
Figure \ref{fig_varyC_woPret} shows the performance of \texttt{FSL}, \texttt{FL} and \texttt{DS} when they start from a randomly initialized model instead of a pretrained one as in Figure~\ref{fig_emnist_varyC}. Clearly, in this case, the acceleration provided by SL is much more significant, even in the IID cases, in which \texttt{DS} offers little to no benefits as one would expect.
\begin{figure}[!tbh]
    \centering
    \includegraphics[width=1.5in]{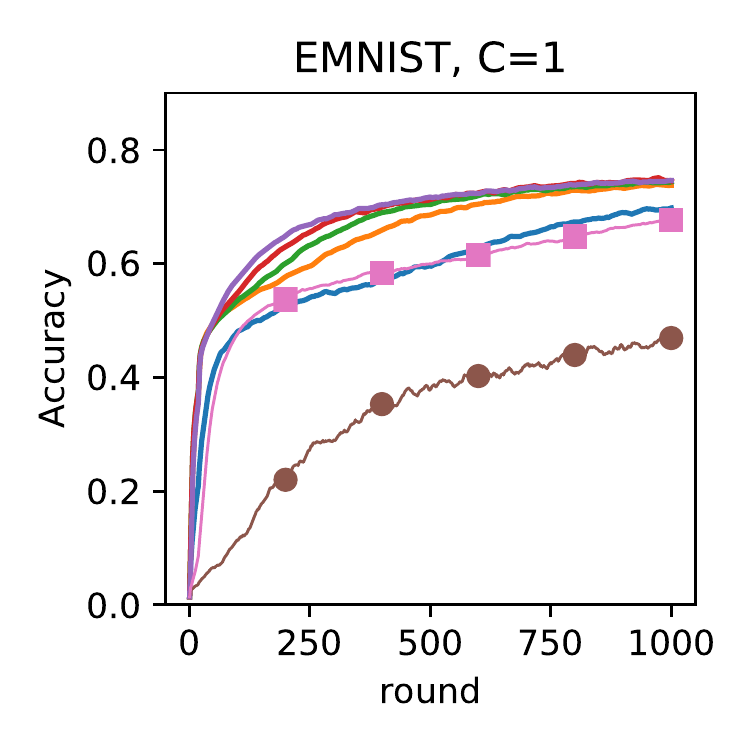}
    \includegraphics[width=1.5in]{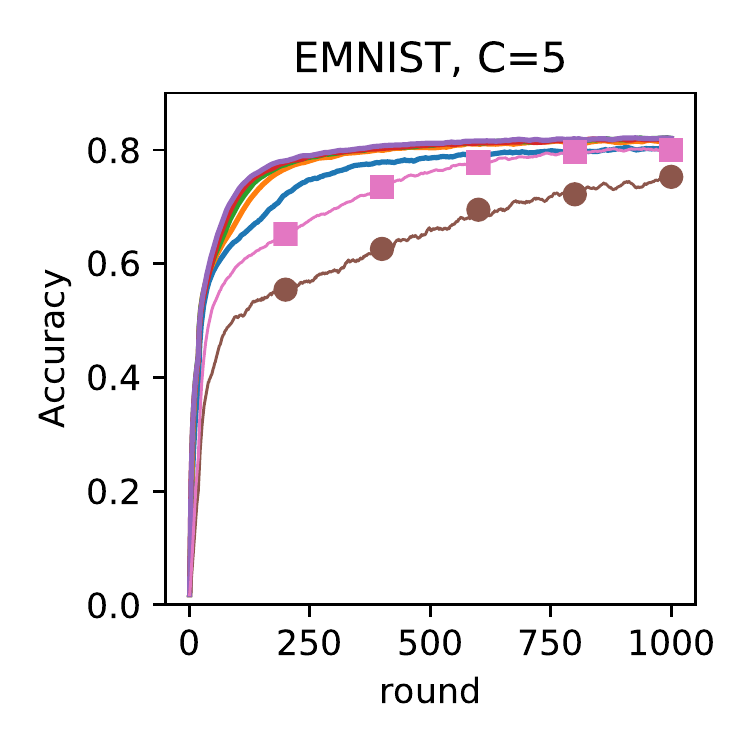}
    \includegraphics[width=1.5in]{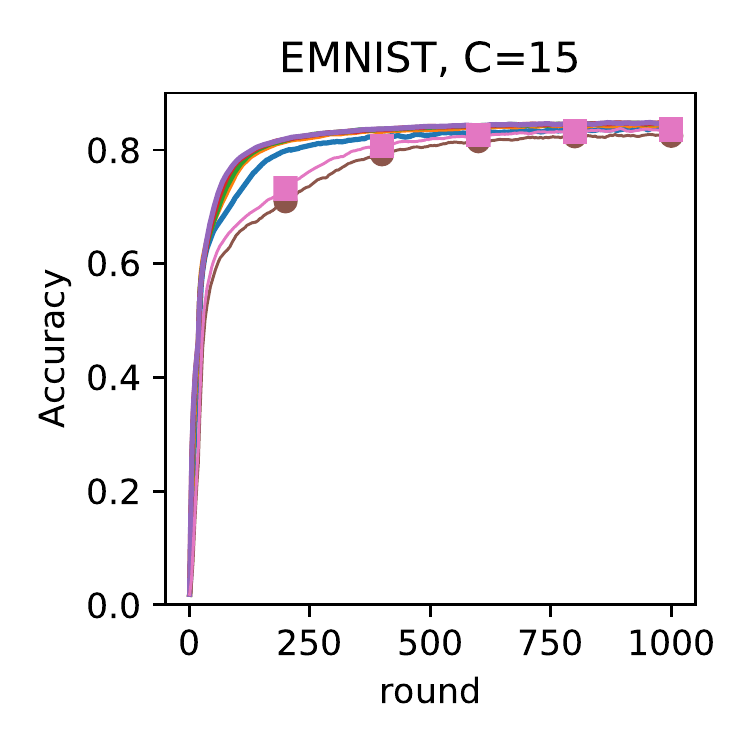}
    \includegraphics[width=1.5in]{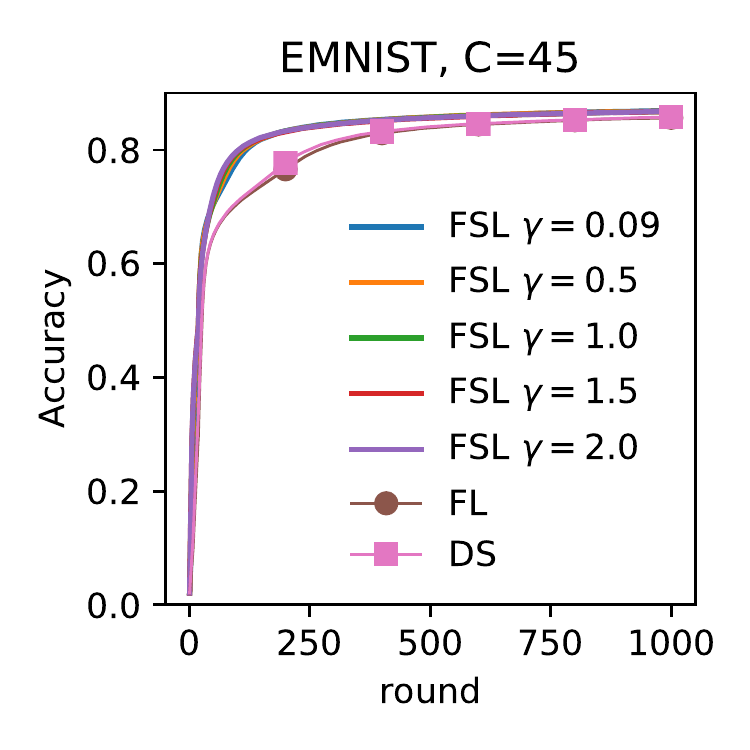}

    \vskip -0.1in
    \includegraphics[width=1.5in]{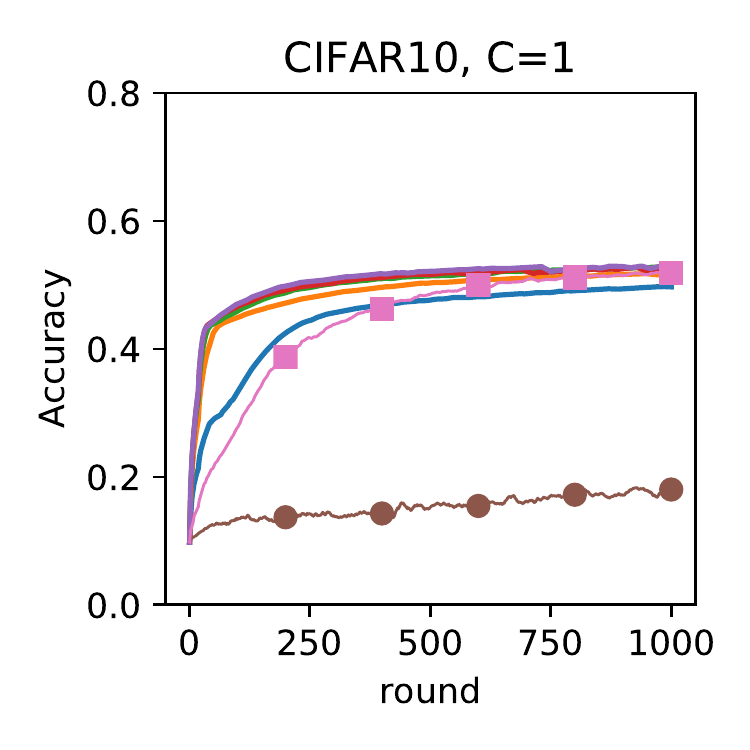}
    \includegraphics[width=1.5in]{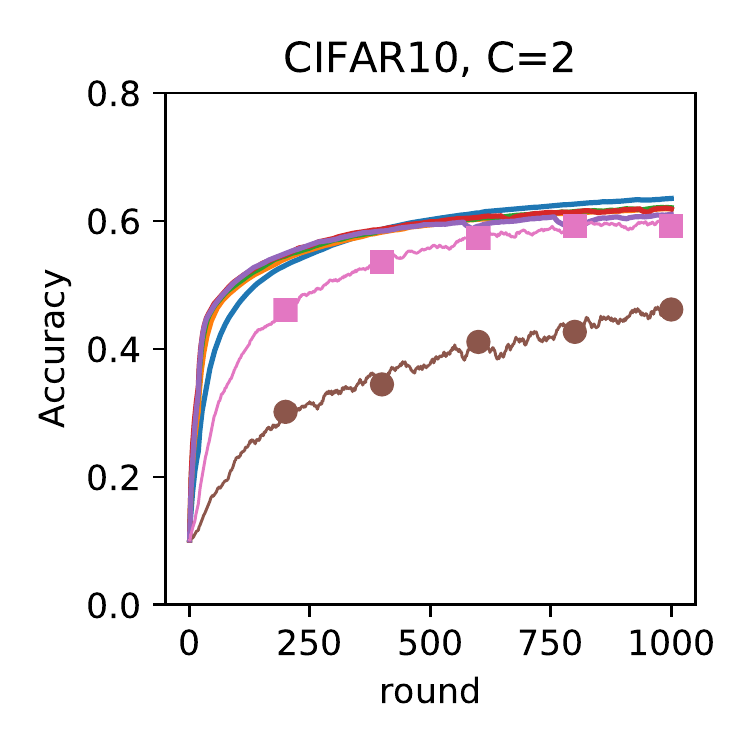}
    \includegraphics[width=1.5in]{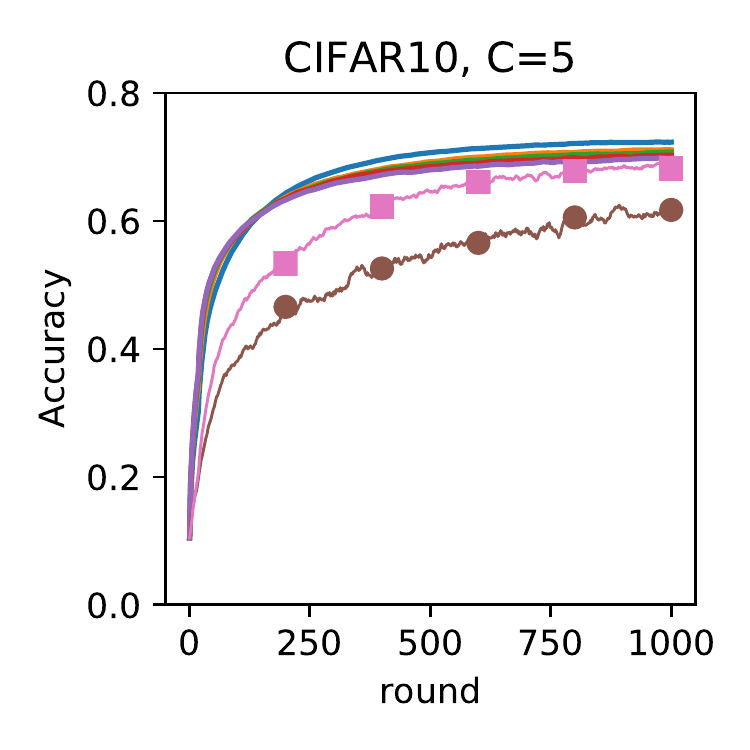}
    \includegraphics[width=1.5in]{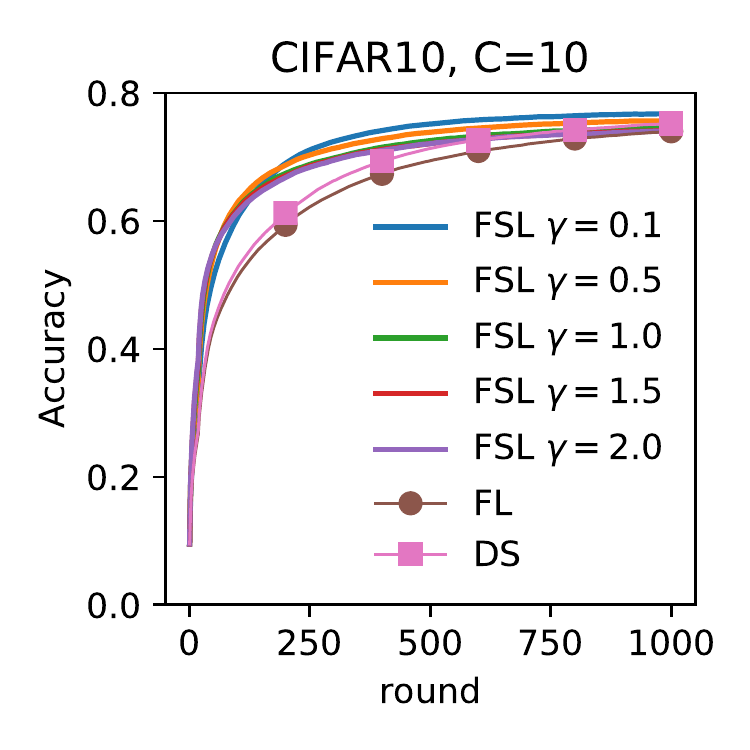}
    
    \vskip -0.15in
    \caption{Test accuracy of \texttt{FSL}, \texttt{FL} and \texttt{DS} when not using server pretrained model. Here, $n_0 = 225$, $S = 5$ and $\eta_l = 0.01$ in EMNIST experiments and $n_0 = 500$, $S = 4$, $\eta_l = 0.01$ for CIFAR-10.}
    \label{fig_varyC_woPret}
\end{figure}

\subsubsection{Comparison with Non-incremental SL}\label{subsec_compare_FSL_nonIncremental}
Figure~\ref{fig_wFSLp_varyCS} compares the performance of \texttt{FSL}, \texttt{DS} and the non-incremental version of SL, denoted by \texttt{FSL-p}, when varying $C$ and $S$. Clearly, \texttt{FSL-p} is slightly worse than \texttt{DS} while  \texttt{FSL} significantly outperforms in all cases. A similar conclusion can be drawn as we vary $\gamma$ as shown in Figure~\ref{fig_acc_rise_SLp_EMNIST}. 

\begin{figure}[!tbh]
    \centering
    \includegraphics[width=1.5in]{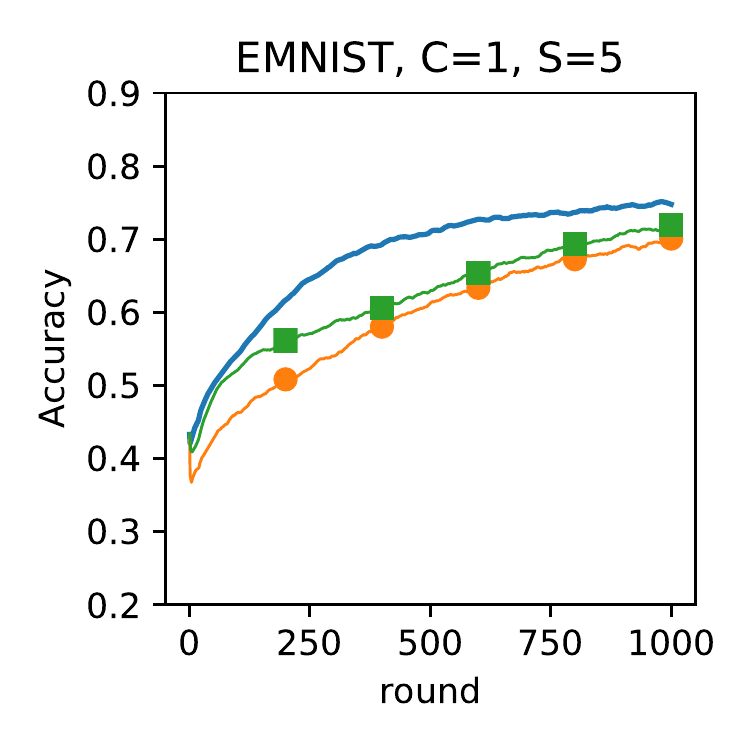}
    \includegraphics[width=1.5in]{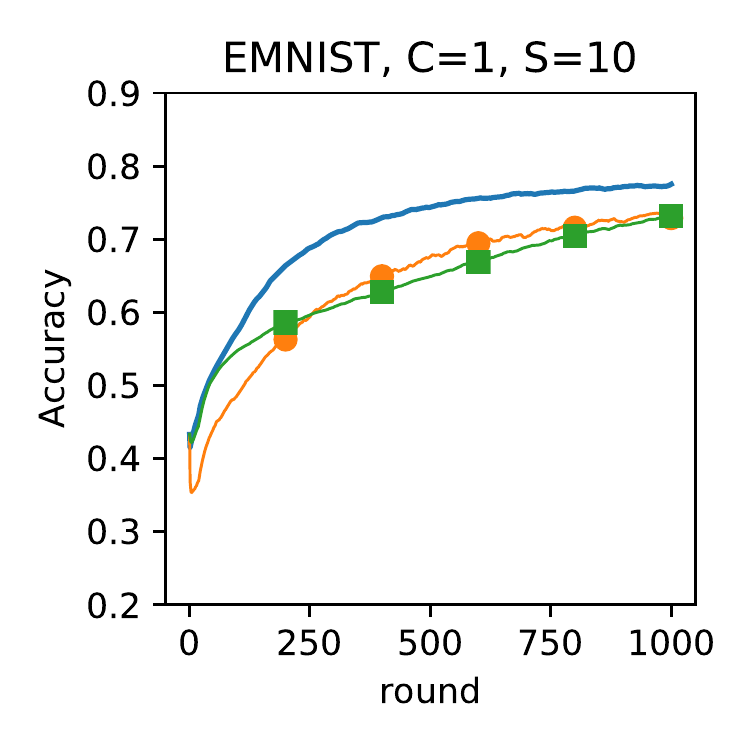}
    \includegraphics[width=1.5in]{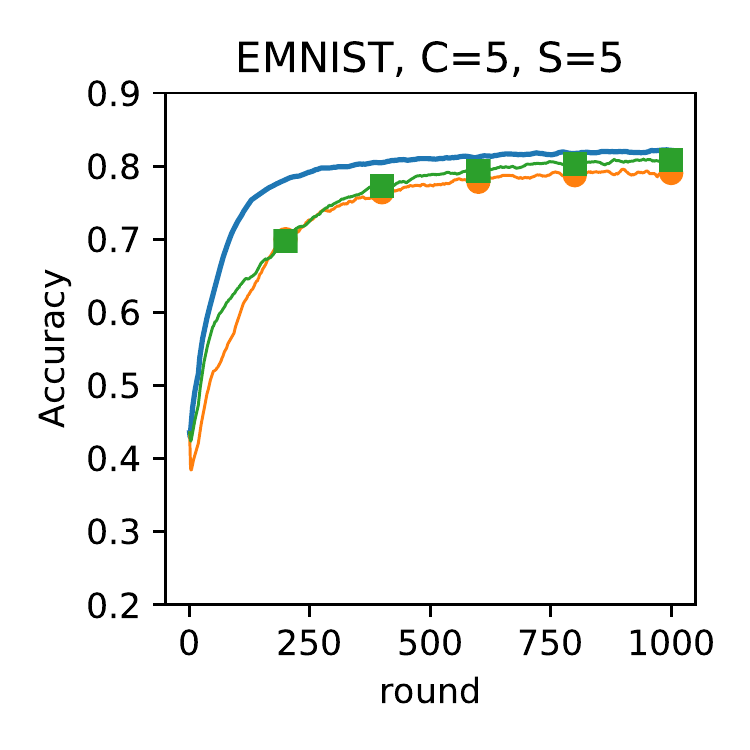}
    \includegraphics[width=1.5in]{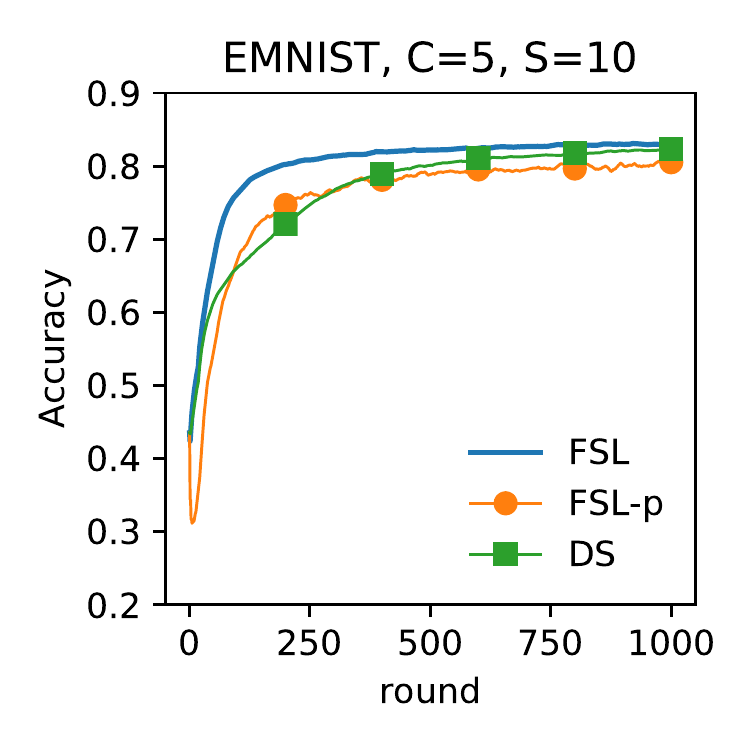}

    \vskip -0.1in
    \includegraphics[width=1.5in]{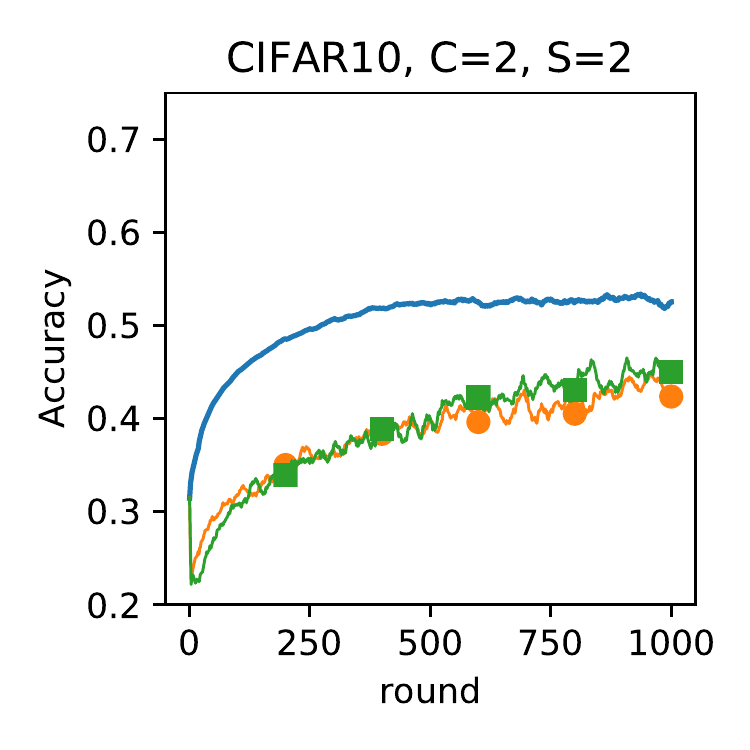}
    \includegraphics[width=1.5in]{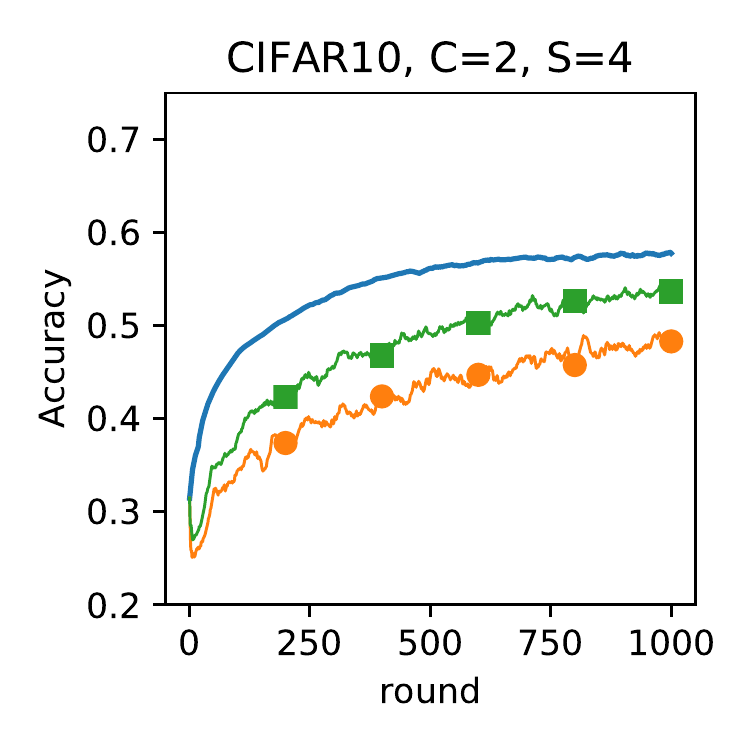}
    \includegraphics[width=1.5in]{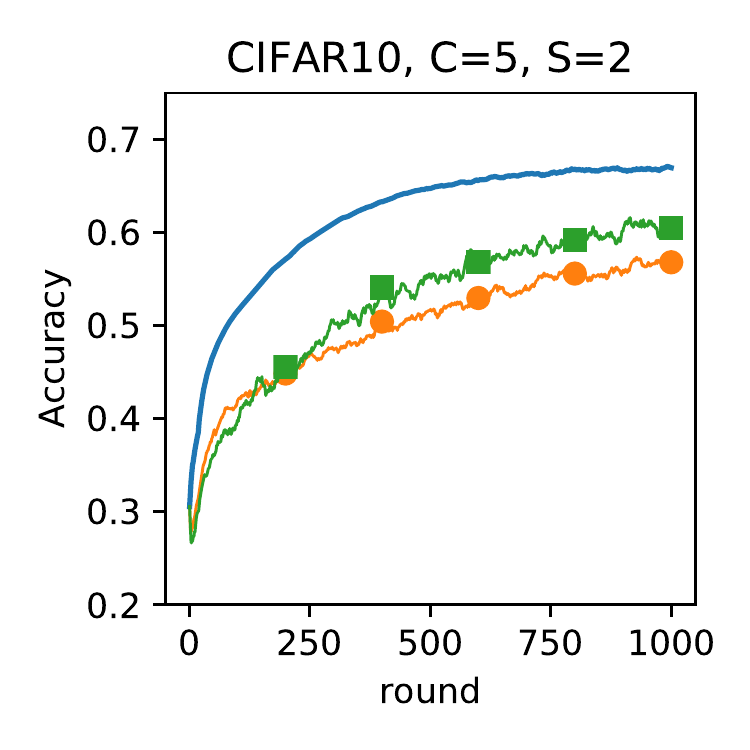}
    \includegraphics[width=1.5in]{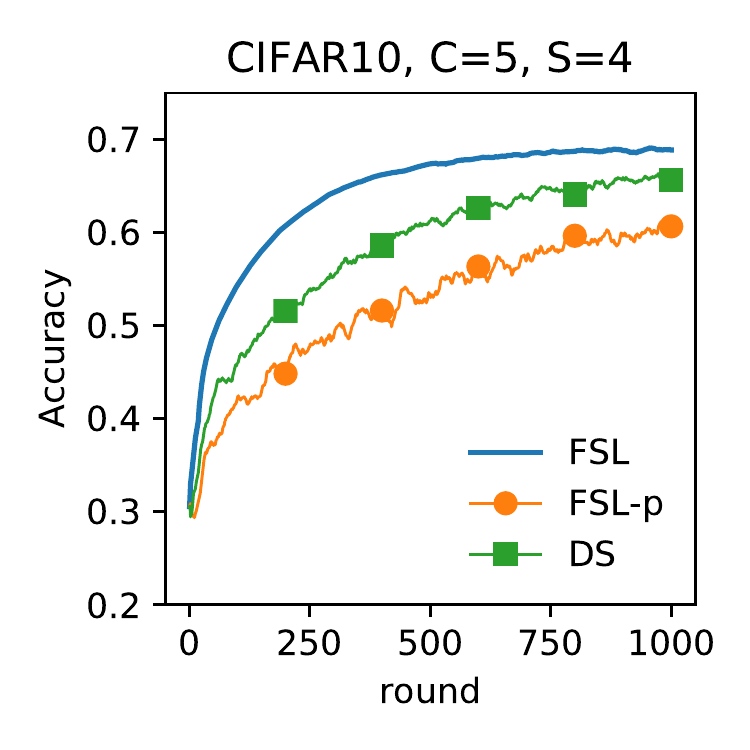}
    \vskip -0.1in
    \caption{Test accuracy of \texttt{FSL}, \texttt{FSL-p} and \texttt{DS}, where $n_0 = 225$, $\eta_l = 0.01$, $\gamma = 1$ for EMNIST, and $n_0 = 200$, $\eta_l = 0.01$, and $\gamma = 1$ for CIFAR-10.}
    \label{fig_wFSLp_varyCS}
    \vskip -0.1in
\end{figure}

\begin{figure}[!t]
    \centering
    \includegraphics[scale=0.4]{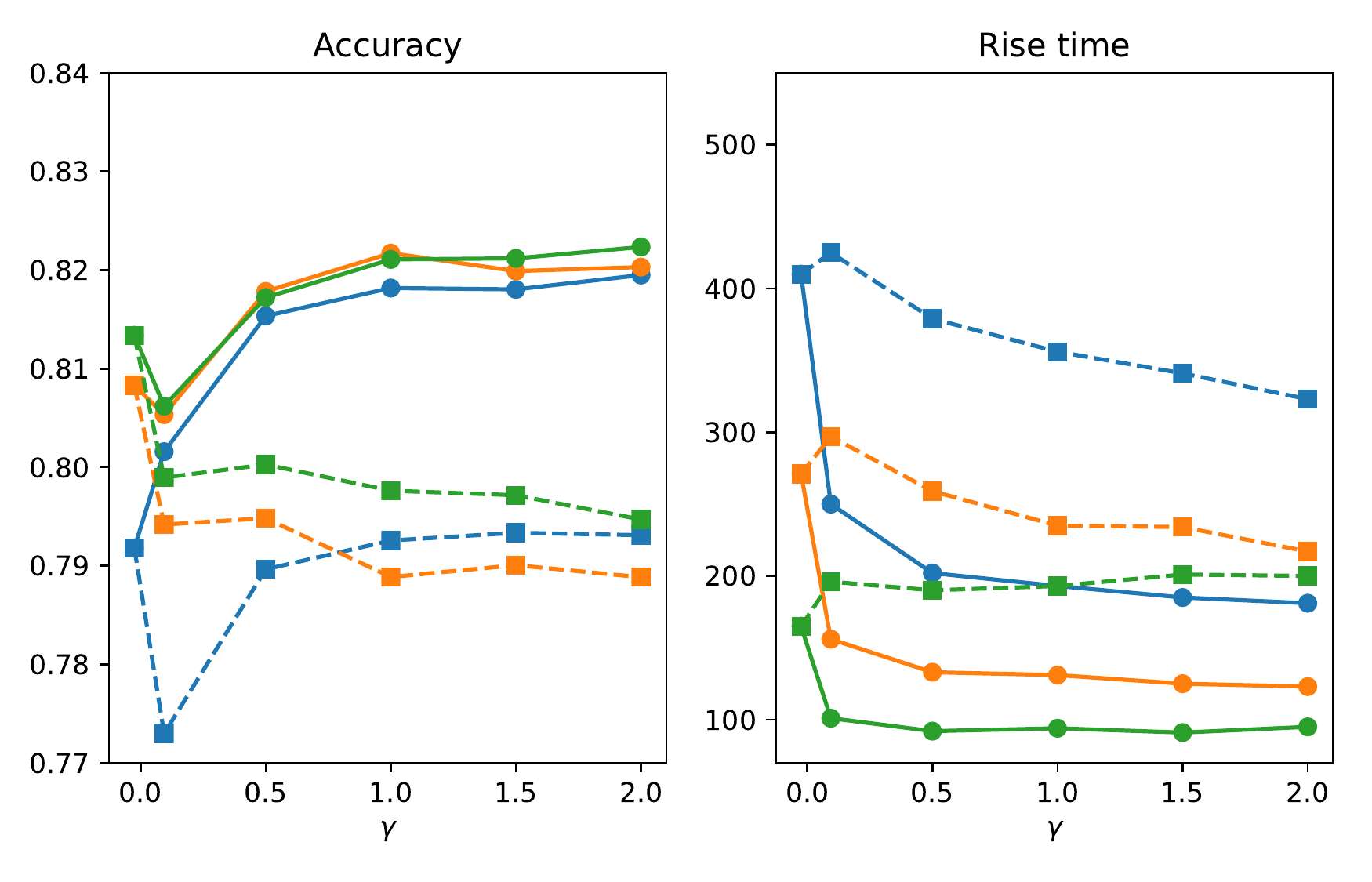}
    \hspace{2ex}
    \includegraphics[scale=0.4]{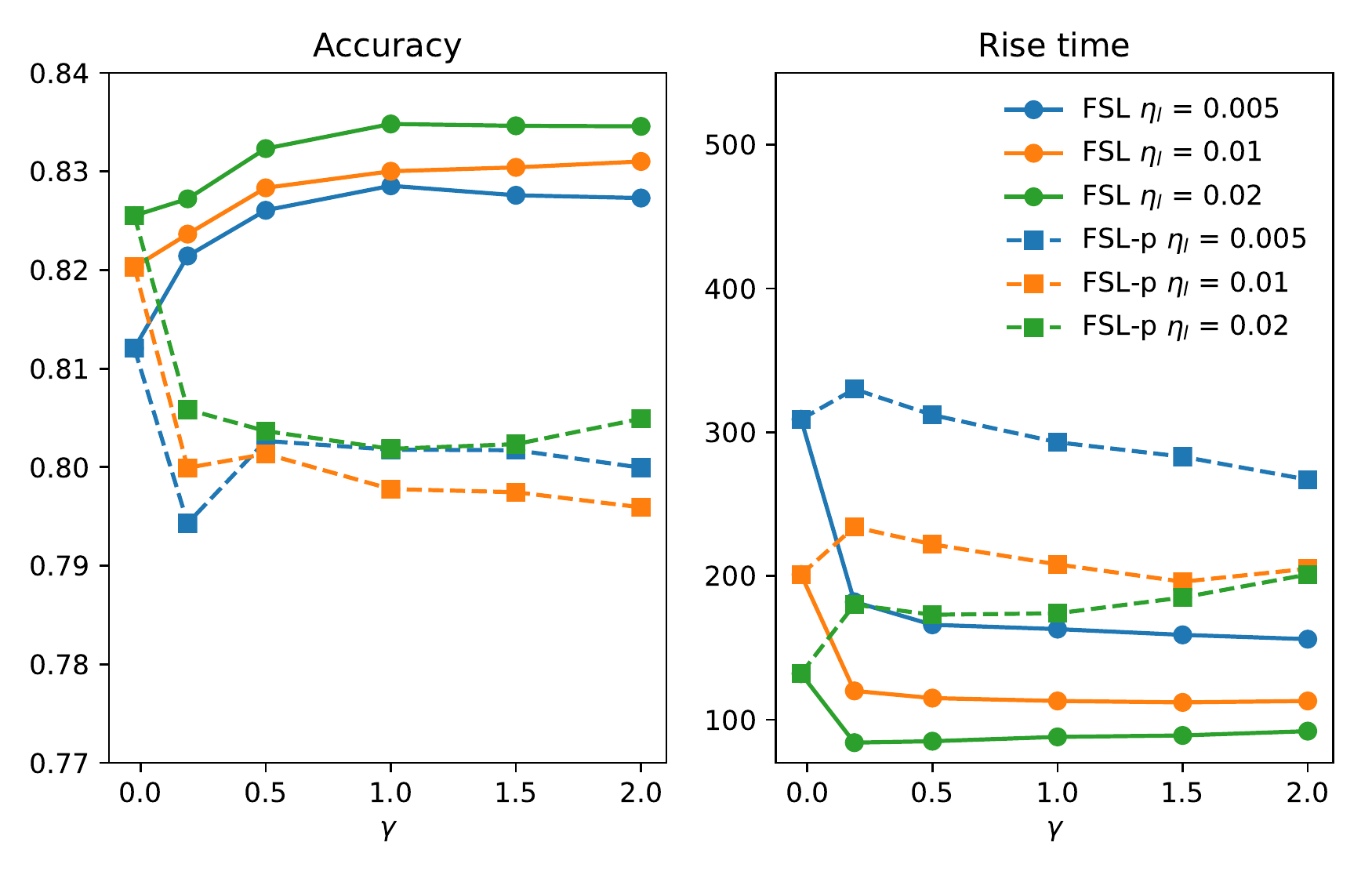}
    \vskip -0.05in
    \centerline{(a) EMNIST with $n_0=225$ \hspace{1.6in} (b) EMNIST with $n_0=450$}
    
    \includegraphics[scale=0.4]{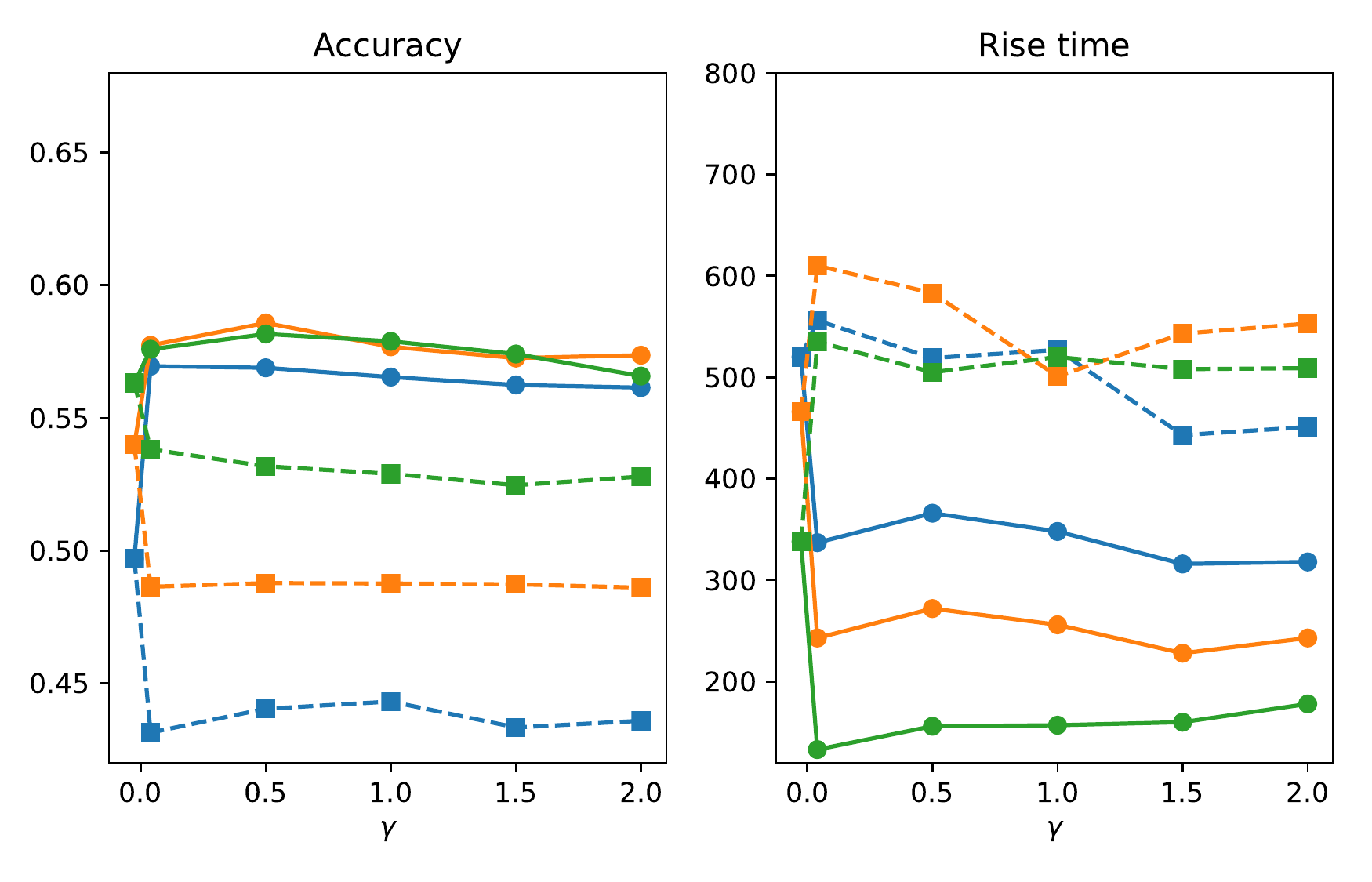}
    \hspace{2ex}
    \includegraphics[scale=0.4]{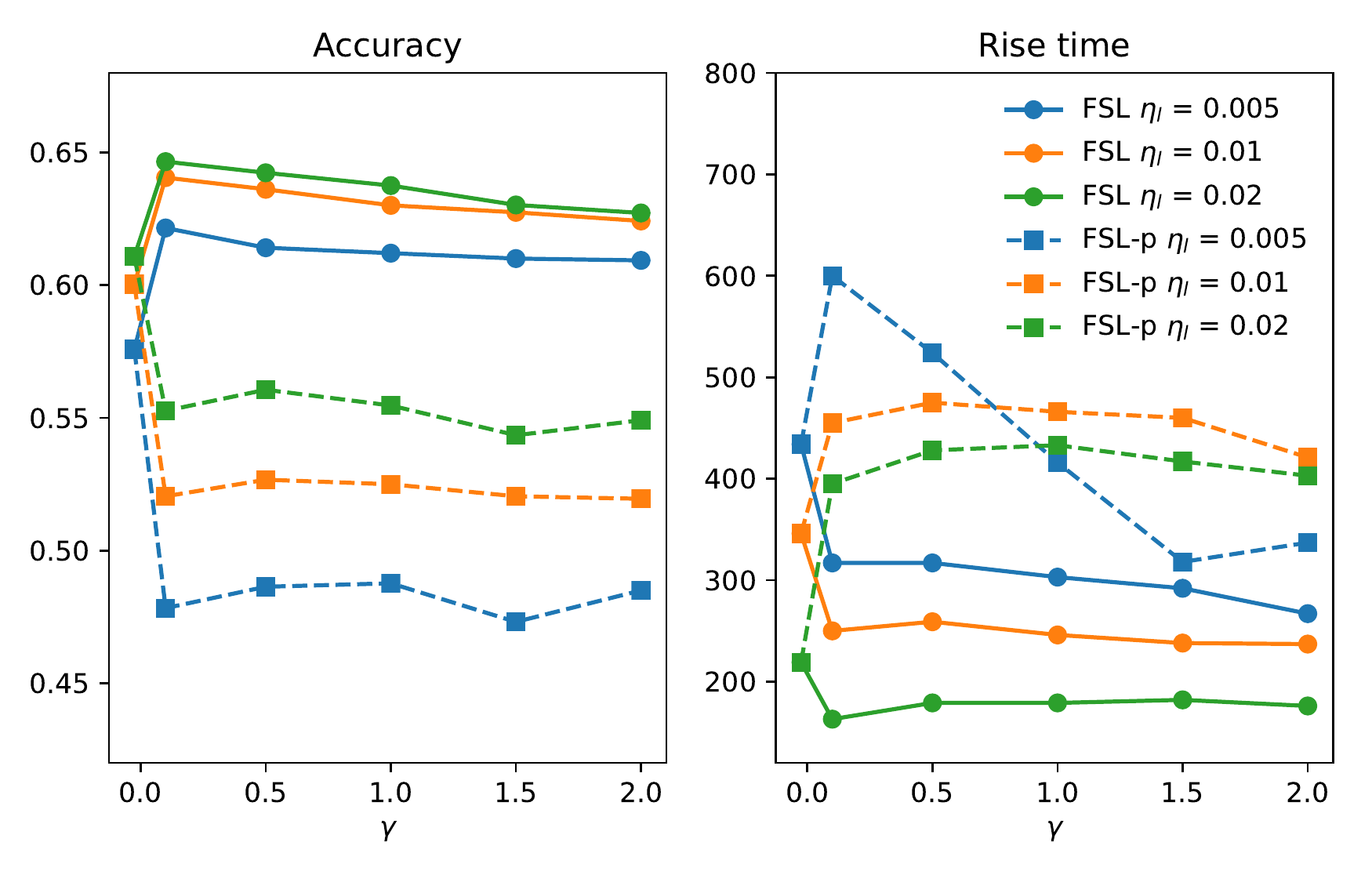}
    \vskip -0.05in
    \centerline{(c) CIFAR-10 with $n_0=200$ \hspace{1.6in} (d) CIFAR-10 with $n_0=500$}
    \caption{Comparison between \texttt{FSL}, \texttt{FSL-p}, and \texttt{DS} (shown at $\gamma = 0$). Here, $(C,S)=(5,5)$ for EMNIST and $(2,4)$ for CIFAR-10.}
    \label{fig_acc_rise_SLp_EMNIST}
\end{figure}
\begin{figure}[!t]
    \centering
    \includegraphics[scale=0.5]{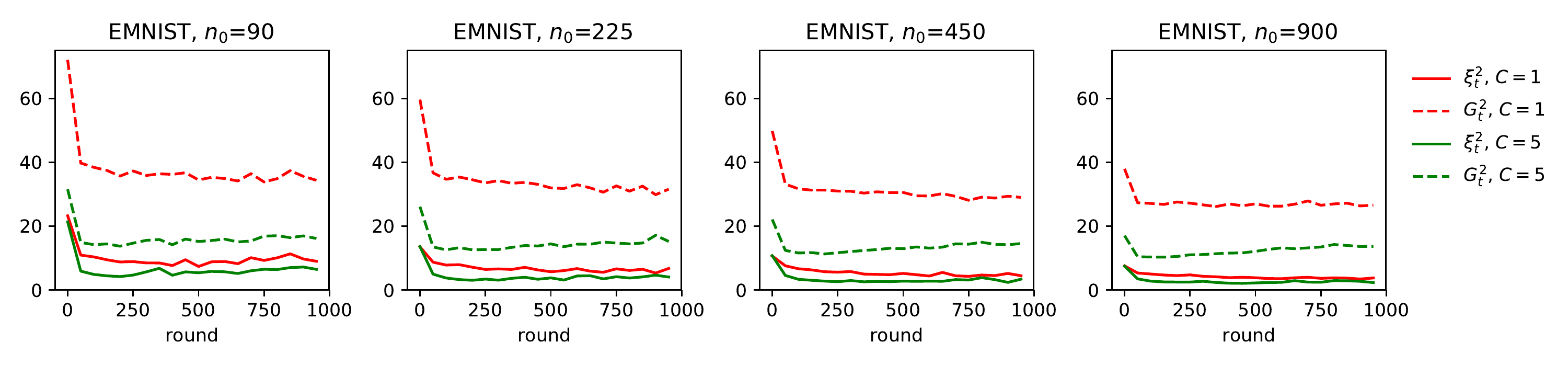}
    
    \vskip -0.1in
    \includegraphics[scale=0.5]{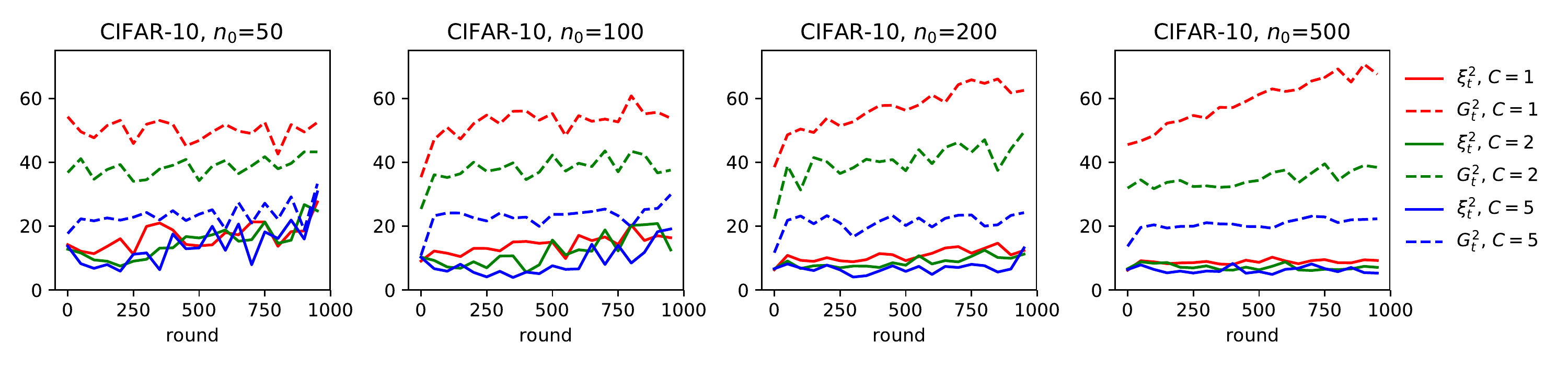}
    \vskip -0.1in
    \caption{Plots of $\xi_t^2 = \|\nabla f_0(x_t) - \nabla F(x_t)\|^2$ and $G_t^2 = \frac{1}{N} \sum_{i\in [N]} \|\nabla f_i(x_{t})-  \nabla F(x_t)\|^2$.}
    \label{fig_nonIIDness}
\end{figure}

\subsubsection{Quantifying the Non-IIDness of Clients and Server}\label{subsec_nonIIDness_bound}
In general, it is difficult to obtain uniform bounds $\bar{\xi}$ and $G$ in Assumptions~\ref{assm_globalGrad}--\ref{assm_server_data}. Thus, in Figure~\ref{fig_nonIIDness}, we show the following two related quantities in some of our experiments: $\xi_t^2 = \|\nabla f_0(x_t) - \nabla F(x_t)\|^2,$ and $\textstyle G_t^2 = \frac{1}{N} \sum_{i\in [N]} \|\nabla f_i(x_{t})-  \nabla F(x_t)\|^2$ for $t=0, 50, 100, \ldots 1000.$ 
First, it is clear that increasing $C$ reduces the non-IIDness considerably in our experiments with both datasets. 
Second, in most cases and on average, $\xi_t^2$ is much smaller than $G_t^2$ and is improved when $n_0$ increases. Third, perhaps somewhat surprisingly, even having access to 2 or 5 samples per label ($n_0 = 90$ and $225$) in EMNIST dataset already offers \texttt{FSL} a significant advantage to combat the  non-IIDness of clients' data. To have a similar level of benefit for the case of CIFAR-10, many more samples are needed for the server's data. This is another indication besides final accuracy that in our experiments, EMNIST dataset is easier to learn from even though it has more label classes than CIFAR-10.

\subsubsection{Further Comparison with FedDyn and SCAFFOLD}\label{subsec_compare_FSL_SCAFFOLD}
Figures \ref{fig_cifar_N100} and \ref{fig_emnist_N450} show the test accuracy (averaged over 3 runs) of \texttt{FSL}, \texttt{FSLsyn}, \texttt{FedDyn}, and  \texttt{SCAFFOLD} with server step size $\eta_g = \sqrt{S}$ and $T=1000$. \texttt{FSL}, \texttt{FSLsyn} not only has comparable (if not better) final accuracy than \texttt{SCAFFOLD} in most cases but also achieves higher initial training acceleration.

\begin{figure}
    \centering
    \includegraphics[scale=0.5]{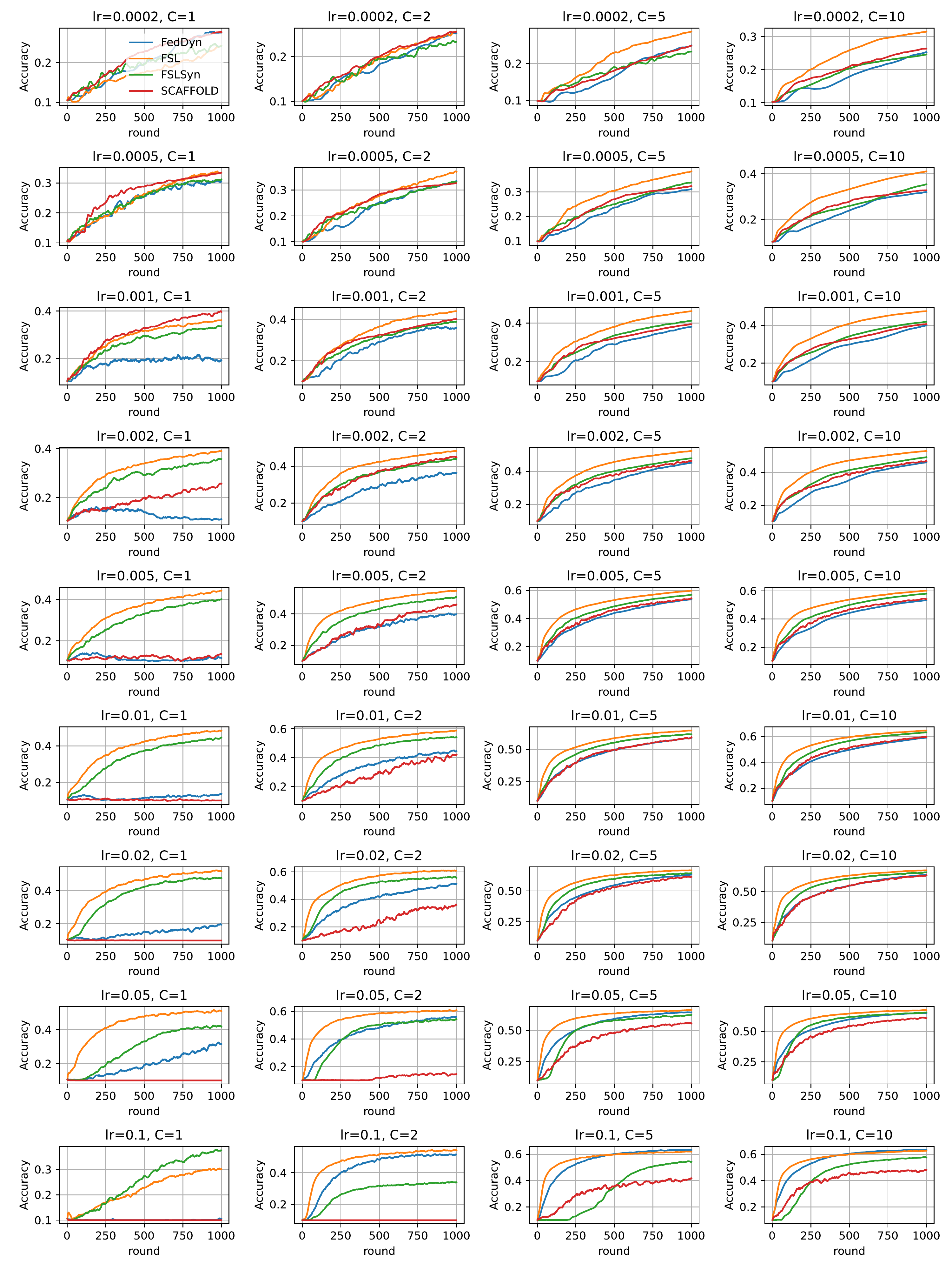}
    \caption{Test accuracy of \texttt{FSL}, \texttt{FSLsyn}, \texttt{FedDyn}, and  \texttt{SCAFFOLD} in CIFAR-10 experiments when varying local learning rate lr$=\eta_l$ and number of label classes $C$ each client has.}
    \label{fig_cifar_N100}
\end{figure}

\begin{figure}
    \centering
    \includegraphics[scale=0.5]{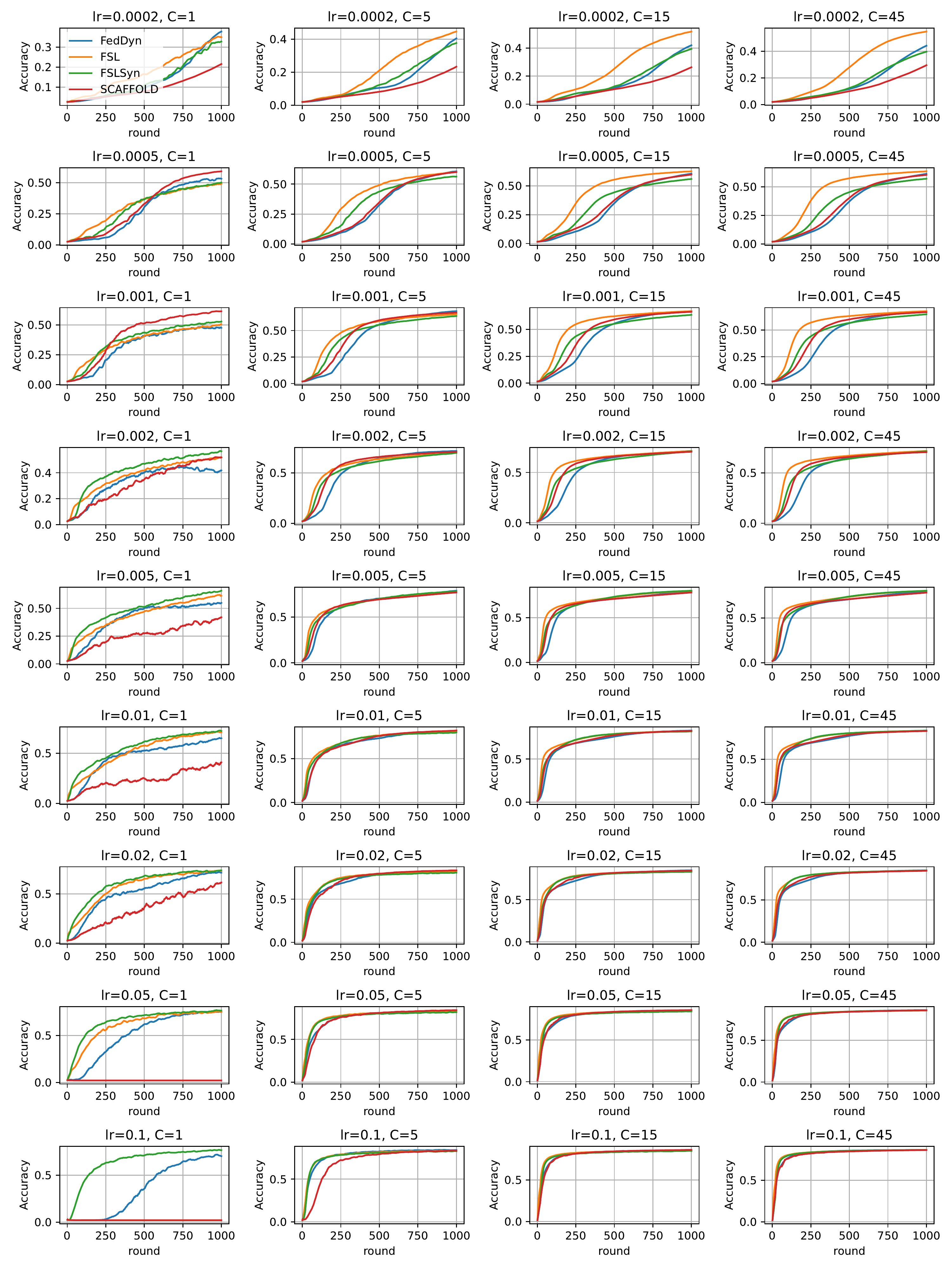}
    \caption{Test accuracy of \texttt{FSL}, \texttt{FSLsyn}, \texttt{FedDyn}, and  \texttt{SCAFFOLD} in EMNIST experiments when varying local learning rate lr$=\eta_l$ and number of label classes $C$ each client has.}
    \label{fig_emnist_N450}
\end{figure}

\end{appendices}

\end{document}